\documentclass{article}

\usepackage{arxiv}

\usepackage{graphicx}%
\usepackage{multirow}%
\usepackage{amsmath,amssymb,amsfonts}%
\usepackage{amsthm}%
\usepackage{mathrsfs}%
\usepackage[title]{appendix}%
\usepackage{xcolor}%
\usepackage{textcomp}%
\usepackage{manyfoot}%
\usepackage{natbib}
\usepackage{booktabs}%
\usepackage{algorithm}%
\usepackage{algorithmicx}%
\usepackage{algpseudocode}%
\usepackage{listings}%
\usepackage{tikz}%
\newtheorem{theorem}{Theorem}%  meant for continuous numbers
\newtheorem{proposition}{Proposition}% to get separate numbers for theorem and proposition etc.

\newtheorem{remark}{Remark}%

\newtheorem{corollary}{Corollary}%
\newtheorem{lemma}{Lemma}%

\usepackage[nameinlink]{cleveref} % make the whole “Figure/Theorem …” part clickable
\usepackage{mathrsfs} 
\usepackage{ wasysym }
\usepackage{threeparttable}
\usepackage{multirow}
\usepackage{booktabs}
\usepackage{xargs}                      % Use more than one optional parameter in a new commands
\usepackage[pdftex,dvipsnames]{xcolor}  % Coloured text etc.

\usepackage[colorinlistoftodos,prependcaption,textsize=small]{todonotes}
\newcommandx{\unsure}[2][1=]{\todo[linecolor=red,backgroundcolor=red!25,bordercolor=red,#1]{#2}}
\newcommandx{\change}[2][1=]{\todo[linecolor=blue,backgroundcolor=blue!25,bordercolor=blue,#1]{#2}}
\newcommandx{\info}[2][1=]{\todo[linecolor=OliveGreen,backgroundcolor=OliveGreen!25,bordercolor=OliveGreen,#1]{#2}}
\newcommandx{\imp}[2][1=]{\todo[inline, linecolor=Plum,backgroundcolor=Plum!25,bordercolor=Plum,#1]{#2}}

\usepackage[utf8]{inputenc}
\DeclareUnicodeCharacter{221E}{$\infty$}

\newcommand{\ie}{\textit{i.e. }}
\newcommand{\eg}{\textit{e.g. }}
\newcommand{\etc}{\textit{etc. }}

\newcommand\num{\stepcounter{equation}\tag{\theequation}}

\newcommand{\kkx}{ \tilde{\kk}_{\tx_j} }
\newcommand{\hkkx}{\hat{\tilde\kk}_{\tx_j}}

\newcommand{\hil}{\mathbb{H}}

\newcommand{\kk}{\mathbf{k}}

\newcommand{\sign}{\operatorname{sign}}

\newcommand{\R}{\mathbb{R}}

\newcommand{\D}{\mathcal{D}}

\newcommand{\E}{\mathbb{E}}

\newcommand{\err}{\mathcal{E}}

\newcommand{\tr}{\operatorname{Tr}}
\newcommand{\dg}{\operatorname{diag}}

\newcommand{\ty}{\Tilde{y}}
\newcommand{\tx}{\Tilde{x}}

\newcommand{\PP}{\mathbb P}

\Crefname{proposition}{Proposition}{Propositions}
\crefname{proposition}{proposition}{propositions} % Match environment name

\newcommand{\domX}{\mathcal{X}}  % domain of X
\newcommand{\kfn}{{\rm k}}  % kernel function k(•,•)
\title{On Kernel Eigen-alignments of KRR: Reconstruction and Generalization}
\author{
 Yang Liu \\
  Galisano College of Coumputing and Information Science\\
  Rochester Institute of Technology\\
  Rochester, NY\\
  \texttt{yl4070@rit.edu} \\
   \And
    Ernest Fokoue\\
    School of Mathematics and Statistics\\
  Rochester Institute of Technology\\
  Rochester, NY \\
  \texttt{epfeqa@rit.edu} \\
  \And
 Richard Lange\\
  Galisano College of Coumputing and Information Science\\
  Rochester Institute of Technology\\
  Rochester, NY \\
  \texttt{rdlvcs@rit.edu} \\
  \And
 Daniel Krutz\\
  Galisano College of Coumputing and Information Science\\
  Rochester Institute of Technology\\
  Rochester, NY \\
  \texttt{dxkvse@rit.edu} \\
}
\begin{document}
\maketitle

\begin{abstract}
This paper investigates the critical role of eigenalignments between the kernel matrix and learning targets in achieving robust generalization in learning problems. We establish a direct connection between generalization performance in kernel methods and the estimation of eigenvectors and eigenvalues of matrices, offering a more intuitive understanding compared to prior work with minimal assumptions. 
We also show that, since the prediction task in KRR is essentially the weighted sum of eigenvectors/singular vectors, by analyzing how much error can be caused by perturbations to the kernel matrix, we can then derive a bound on this generalization error using the estimation stability of matrix eigenvalues and eigenvectors.
Compared with previous work, our analysis concentrates on finite-sample settings and on the generalization error arising from having a suboptimal finite training set. 
Our findings reveal that in kernel methods, as long as the kernel is of high rank, the near-zero reconstruction error can be trivially obtained, implying that the reconstruction error will have limited predictive power for generalization. 
Finally, we establish a generalization bound from an eigenvalues/eigenvectors estimation perspective, showing that strong generalization requires increasing eigenvector alignment, eigenvalue magnitude, or gaps between consecutive eigenvalues.
\end{abstract}
% \keywords{Learning Theory, Eigen-alignment, KRR}


\maketitle

\newpage
\section{Introduction}\label{sec:intro}

This paper investigates the critical role of eigenalignments between the kernel matrix and learning targets in achieving satisfactory generalization performance in learning problems. These eigenalignments, which represent the inherent properties of the data and the selected kernel, are frequently overlooked in conventional learning theory. We establish a direct connection between generalization in kernel ridge regression and the estimation of eigenvectors and eigenvalues. Our work is based on the foundational work of \cite{canatar2020, jacot2020, simon2021}. Compared to previous work, our analysis focuses on the finite sample regime, thus avoiding issues related to the slow convergence rate that has been pointed out by \cite{simon2021, wei2022}. Based on a novel perspective of matrix perturbation theory, our analysis makes minimal assumptions and provides a more intuitive and clearer understanding of the mechanism of eigenalignments through the lens of eigenvectors/eigenvalues estimations, revealing previously hidden or unnoticed implications. We show that for a given kernel matrix, the generalization error can depend non-trivially on the learning targets. Additionally, in high dimensions, non-parametric kernels can achieve nearly zero reconstruction error, rendering the relationship between reconstruction error and generalization error negligible in these cases. We demonstrate this using synthesized learning targets, as classical machine learning datasets tend to contain well-behaved labels, which can lead to spurious correlations. For a given kernel matrix with enough rank, we demonstrate that achieving a near-zero training loss is nearly always feasible. This finding implies that the training error has limited utility in predicting the generalization error.

Furthermore, we establish a generalization bound based on the predictive estimation of finite sample eigenvalues and eigenvectors. We show that a good generalization requires either an increase in the alignment of eigenvectors, an increase in the eigenvalue, or a widening of the gap between adjacent eigenvalues. Our results show that the regularization parameter often has a minimal impact. This suggests that its primary role is to mitigate numerical instability, rather than its traditionally perceived function of reducing hypothesis space in the non-parametric space. When regularization does influence generalization, it is due to its alignment with trailing eigenvectors, occurring when the corresponding eigenvalue is comparable in value to the regularization parameter. We show that the regularization is more related to salvaging a bad performance than improving a good performance. Hence, we can summarize this using the common bias-variance terminology: bias is attributed to the inability of the top eigenvectors to represent the learning targets, while variance arises from the requirement for the learning targets to be represented by the trailing eigenvectors. These findings align with the existing literature \eg \cite{hastie2022, liang2018}. 
Our findings also suggest that overfitting in the non-parametric kernel ridge regression setting is essentially related to the alignments of trailing eigenvectors, which are dominated by noise. As a result, this implies the true nature of overfitting, in this case, is the kernel lacking the capacity to differentiate noise and information due to either the data being fundamentally too noisy or the kernel choice not being optimal. 

% Additionally, we explore the relationship between eigenalignment and the double descent phenomenon. Our empirical findings indicate that double descent can manifest when the target variables exhibit a correlation with trailing eigenvectors.
Our analysis emphasizes the effects of the choice of learning targets on the model's generalization behavior, particularly in the context of Kernel Ridge Regression (KRR). \Cref{fig:targetmatters} presents two key aspects of our findings: we demonstrate that the generalization error (on a logarithmic scale) increases as the rank of the eigenvector with which the learning targets are aligned increases. Regularization, on the other hand, only has a slightly positive impact (reducing the generalization error) when the learning targets are aligned with trailing eigenvectors. The details of the experiments can be found in \Cref{sec:exp1}.

% We carefully analyze the intricate relationships among reconstruction error, overfitting, and robust generalization.
% In the high-dimensional non-parametric kernel methods regime,
% However, we believe that the effects of the learning targets may extend beyond KRR.  Intuitively, this concept can be illustrated using a straightforward example involving 1-dimensional data.  

\begin{figure}[!ht]
    \centering
    \includegraphics[width=0.75\linewidth]{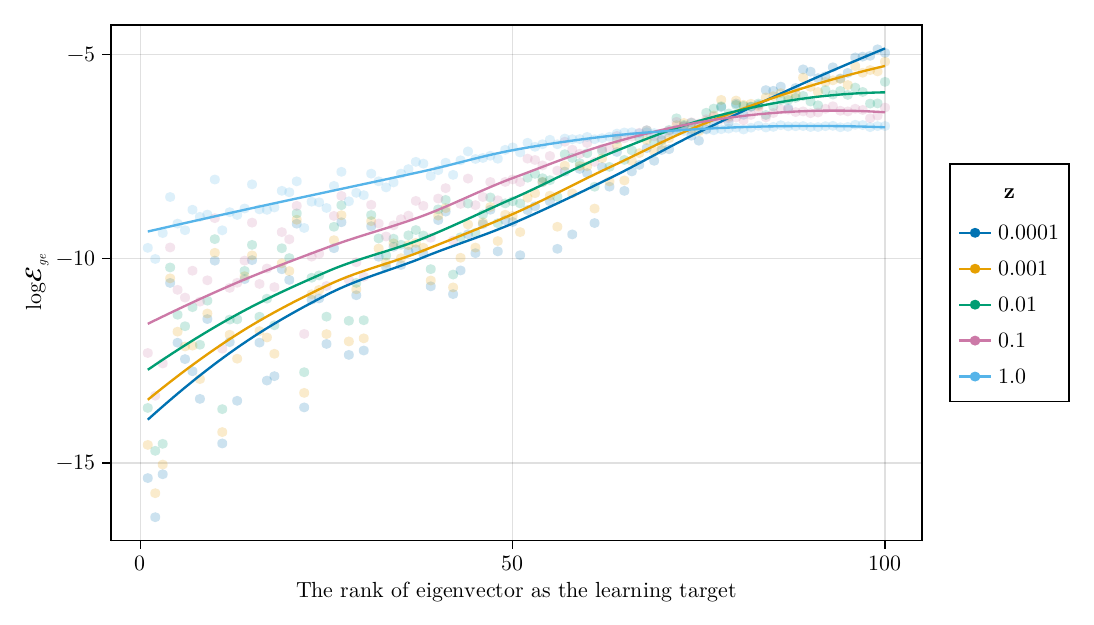}
    \caption{Learning targets of the task affect the generalization performance. Learning targets aligned with top eigenvectors (lower rank) tend to have lower generalization errors. Regularization tends to have a negative impact on generalization error for most learning targets, except those aligned with trailing ones (high-rank eigenvectors).}
    \label{fig:targetmatters}
\end{figure}

\subsection{Problem Setup}\label{sec:prob}

In this paper, we study the kernel selection problem in the context of large-dimension kernel ridge regression, which is also referred to as kernel LS-SVM in the literature. Our work is in the same vein as recent research in kernel alignment \eg \cite{jacot2020, canatar2020}.

We consider the typical setting of Kernel Ridge Regression (KRR). Denote the training dataset $\mathcal{D} = \{X, y\}$ of size $n$ and let $K_n$ represent the kernel constructed from the input data sample $X$. Each entry in $K_n$ is evaluated using $\kfn(x_i, x_j)$, where a kernel function $\kfn : \domX \times \domX \rightarrow \R$, $K_n$ represents the $n \times n$ matrix of kernel evaluations $\kfn(x_i, x_j)$, and KRR learns a function that maps any given $x$ to a prediction based on a weighted sum of kernel evaluations on the training data:
$$\hat{f}_n(x) =\sum_{i=1}^n \alpha_i \, \kfn(x,x_i) \, ,$$
where $(\alpha_1, \ldots, \alpha_n)^\top = (K_n+ zI_n)^{-1} y$ and $z \in \R^+$ is the ridge regularization parameters. Although our analysis can be applied to all kernel functions, our focus is leaning using inner product kernels in high-dimension cases, as we consider them to be more relevant in modern machine learning settings. As shown in \citet{couillet2016, couillet2021, liao2019}, inner product kernels are more effective than distance-based kernels such as Gaussian kernels: distance-based kernels minus the diagonal has an asymptotic rank of one. By inner product kernels, we specifically refer to a kernel matrix calculated using
$$
K_{i,j} = h(x_i^\top x_j),
$$
where $h$ is some chosen function. 
We will refer to the training data labels as $y$ learning targets and will use $\hat{y}$ to denote the predictions or reconstructions of the training data, which can be expressed simply as
$$
\hat{y} = K_n (K_n+ zI_n)^{-1} y \, .
$$
The reconstruction error, also known as training loss, is
$$
\err_{re} := \frac{1}{n} \sum_{i=1}^n(y_i - \hat{y_i})^2,
$$
where $y_i$ denotes the $i$th label for $y$. %In the remainder of the paper, \textit{we will refer to this setting as the KRR assumptions.} 
For any pair of observed data $\{\tilde{X}, \ty\}$, the predictions are
$$
\hat{\ty} = \tilde{K} (K_n+ zI_n)^{-1} y \,, 
$$
where $\tilde{K}_{i,j} = \kfn(\tx_i, x_j)$.
We are particularly interested in how the learning target affects the generalization error, which is also quantified using the expected squared error:
$$
\err_{ge} := \E(\ty_j - \hat{\ty}_j)^2,
$$
where $(\tx_j, \ty_j)$ is an unseen data point and $\hat{\ty}_j = f(\tx_j)$. Furthermore, in our analysis, we will extensively utilize the spectral decomposition of the kernel matrix $K_n$, in the sense that $K_n = \sum_{i=1}^n d_i u_i u_i^\top$. We will assume that the eigenvalues are sorted $d_1 \geq d_2 \geq \cdots \geq d_n$. We refer to the first few (small-number) eigenvalues and eigenvectors, characterized by small values, as the top eigenvalues and top eigenvectors, whereas the remaining eigenvalues and eigenvectors are termed trailing eigenvalues and trailing eigenvectors.

\noindent
We list the most used notation in \Cref{tab:notation}. 
% \subsection*{Symbols and Notations}

\begin{table}[!h]
    \centering
    \begin{tabular}{c c }
        \toprule
        Notation & Meaning\\
        \midrule
         $K_n$ & $n\times n$ kernel matrix \\ 
         $\kfn(\cdot, \cdot)$   & Scalar kernel function \\
         $u$  & Eigenvector of the kernel matrix \\
         $d$  & Eigenvalue of the kernel matrix \\
         $z$  & Regularization parameter \\
         $x, y$ & Training input data point and learning targets  \\
         % $a$ & kernel weight/coefficient as in $\hat{f}_n(x) = \sum_{i=1}^n \hat a_i k(x_i, x)$\\
         $[\cdot]_i$ & $i$th data point \\
         $[\cdot]^*$  & Optimal decorator \\ 
         $\hat{[\cdot]}$ &  Prediction decorator for $(x, y)$\\
         $\tilde{x}$, $\tilde{y}$ & Testing data pair\\
         % $\hat{\tilde{y}}$ &  Prediction for test learning targets\\
         $\err_{re}$ & Reconstruction error \\
         $\err_{ge}$ & Generalization error \\
         $\kkx$ &  Row vector from the kernel matrix corresponding to $\tilde{x}_j$\\
         $\PP$ & Projection operator\\
        \bottomrule
    \end{tabular}
    \caption{List of Notations Used}
    \label{tab:notation}
\end{table}

% \subsection*{Main Contributions}

% We properly analyze the non-parameteric kernel metric, or KRR, through the lense of eigen-alignment. 
% Our main findings include the following:
% \begin{enumerate}
%     \item Show that the reconstruction of learning targets in KRR is a weighted sum of eigenvectors of the kernel matrix. 
%     \item Show that, as in non-parametric high-dimensional regime that we can always obtain near zero loss, the training error has little predictive power over generalization. 
%     \item Show that a small predictive error depends on the stable estimation of eigenvectors and eigenvalues, implying that only learning targets aligned with top eigenvectors can generalize well.
%     \item Demonstrate that in high-dimensional regime when learning targets are aligned with trailing eigenvectors, the double descent phenomenon can be observed. 
% \end{enumerate}

\section{Main Results}\label{sec:res}
\subsection{Summary Of Results and Contributions}

We first provide a brief summary of the main results presented in our paper, delegating detailed discussions to the subsequent sections. In the KRR setting, the reconstructions can be written in terms of eigenvalues and eigenvectors of $K_n$:
$$\hat{y} = \sum_{i=1}^n u_i (d_i u_i^\top y)/(d_i + z) \, .$$
In other words, training data reconstructions are simply weighted sums of eigenvectors \cite{canatar2020,jacot2020,simon2021}.
%We refer to this as the eigen-characterization of the KRR.
In equation (\ref{eq:bvdecomp}), we will then introduce a novel decomposition of the reconstruction error:
\begin{align*}
    \| y - \hat{y}\|^2_n  = n\err_{re}= \sum_{i: d_i \neq 0} \frac{z^2 (u_i^\top y)^2}{(d_i + z)^2} + \sum_{i: d_i = 0}(u_i^\top y)^2.
\end{align*}
where $\|\cdot\|_n$ is the empirical norm of $L_2$. We observe that the minimization of the reconstruction error $\err_{re}$ is not correlated with the accuracy of estimating the eigenvalues and eigenvectors. This implies that reconstruction is a rather trivial task in high-dimensional spaces, and the generalization performance of KRR cannot be accurately predicted by the training error alone, in contrast to the methods of \cite{jacot2020}.

We proceed to demonstrate that the key to the generalization error (as opposed to the reconstruction error) lies in an accurate estimation of eigenvalues and eigenvectors. Specifically, we show that for any new data point $\tx_j$, when the learning targets are aligned with the top eigenvectors of the kernel matrix, the predictive estimation error can be bounded by
$$
C_j\sum_{i=1}^n|\PP_{u^*_i}^y|\left( \sum_{k:k\neq i}\frac{2\|\Delta K_n\|}{|d^*_i-d^*_k|} + (1 - \Theta_i) + \frac{\|\Delta K_n\| +2+\sqrt{2- 2\Theta_i} }{d^*_i + z} \right),
$$
where $K^*_n$ is the ``optimal kernel matrix'' of size $n$. Here, $\|\cdot\|$ for the matrix is understood as the operator norm, $\{d^*_i, u^*_i\}$ represents the eigenpairs of $K_n^*$, $C_j$ is a multiplier that depends on the new data point $\tx_j$ but not on the training data, and $\Theta_i$ denotes the estimation alignment of the $i$th eigenvector (which will be introduced formally later).
For any given training set, $K^*_n$ and $K_n$ are fixed. Hence, we notice that in order to achieve good generalization, we need either to increase the alignment of eigenvectors, increase the eigenvalue $d^*_i$, or increase the eigen gap $|d^*_i-d^*_k|$; all of these can be achieved if the learning targets are aligned with the top eigenvectors, \ie $\PP_{u^*_i}^y := u_i^{*\top} y = 0$ for all $i$ except for several top eigenvectors $u^*_i$. 

\subsubsection{Comparison to Previous Work}

There have been many works on the topic of eigen-alignments of KRR; compared to past work, our paper's unique contributions are the following:
\begin{enumerate}
    \item Previous work such as \cite{simon2021, jacot2020} makes Gaussian assumptions to obtain the results: this assumption likely not true in practice \citep{barzilai2023}. Our results require no such distributional assumptions.
    \item Methodologically, we develop a perturbation-based analysis on the kernel alignments that allows us to derive non-asymptotic finite sample bounds. This analysis is facilitated by a novel method of constructing an optimal finite-sample dataset.
    \item We obtain explicit bounds on generalization error in terms of eigen-alignments and spectral gaps (\Cref{thm:gen}) and suggest a simple empirical estimator \cref{eq:ours}). Previous literature contains either convoluted formalization such as \cite{jacot2020}, or hard-to-interpret implicit formulations of generalization error such as \cite{canatar2020, bordelon2020}, obfuscating the underlying mechanisms of generalization. As a result, several key insights are buried in the literature, and some formulations were even misinterpreted by the authors themselves:
    \begin{enumerate}
        \item Contrary to claims of \cite{jacot2020}, we show that reconstruction error or training error has minimal predictive power for generalization performance.
        \item We emphasize that ridge regularization has minimal effects in improving generalization beside the controlling numeral issues.
        \item We highlight these effects in cases of kernel-task misalignment, i.e. cases where targets align with trailing eigenvectors.
    \end{enumerate}
    % \item We have demonstrated in our empirical experiments that existing literature such as \cite{jacot2020, simon2021} significantly overestimates the predictability of trailing eigendirections. 
\end{enumerate}

\subsection{Eigen-characterization of KRR}

We first present a straightforward yet insightful eigen-characterization of the KRR. While the results resemble earlier findings in the literature, such as \cite{jacot2020, canatar2020, simon2021}, we provide a more intuitive interpretation and offer new insights. 
The goal is to understand the training and generalization behavior of KRR in a non-parametric regime, particularly in high dimensions. 
For this purpose, recall the reconstruction operator for $z \in \R^+$,
$$
A = K_n  ( K_n  + zI_n)^{-1} .
$$
then $\hat{y} = A y.$
% During the predictions of the test set, we replace $K$ with $\tilde{K}$, which is the predictive Gram matrix. Then, $\hat{\tilde{y}} = \tilde{A} y$,
Through algebraic manipulation, we have
\begin{align*}
   A &=  K_n  ( K_n  + zI_n)^{-1} \\
   &=  K_n  ( K_n  + zI_n)^{-1}\\
   % A_{kk} &= \tilde{d}_k \tilde{}_{.\ell}_{.k}^\top ( K  + zI_n)^{-1}
   % &= U \Lambda U^\top Q_{K}(z)\\
   &= U\Lambda U^\top \sum_{i=1}^n \frac{u_iu_i^\top}{d_i + z} \\
   &= \sum_{i=1}^n \frac{d_i u_i u_i^\top}{d_i + z} 
\end{align*}
where $\Lambda = \dg(d_1, \dots, d_N)$;
thus, for training predictions
\begin{align*}
    \hat{y} &= A y\\
    &= \left( \sum_{i=1}^n \frac{d_i u_i u_i^\top}{d_i + z} \right) y\\
    &= \sum_{i=1}^n \frac{d_i u_i u_i^\top y}{d_i + z} \\
    &= \sum_{i=1}^n \frac{(d_i \PP_{u_i}^y) u_i}{d_i + z} \num \label{eq:recon}
\end{align*}
where $P$ is the projection operator such that $\PP_{u_i}^y := u_i^\top y$. Therefore, predictions are the weighted sum of eigenvectors, with weight $(d_i\PP_{u_i}^y)/(d_i + z)$. The alignment of the target vector suggests the magnitude of the contribution; the regularization coefficient $z$ on the other hand, determines the directions of the contribution. Poor alignment of the eigenvector with the target vector, despite the large eigenvalue of the kernel matrix, can result in a negligible contribution to the final predictions. Additionally, this implies that the target vector can learn at most the linear combinations of the eigenvectors of the kernel matrix.  This result resembles Proposition 11 of \cite{muandet2016}.

Since $\{u_i\}_i^n$ forms an orthonormal basis, $y = \sum_i \PP_{u_i}^y u_i$ and $d_i/(d_i+z)$ actually adjust the weights to ensure that the predictions are more dependent on the dominant eigenvectors. This approach avoids reliance on purely noisy eigenvectors and incorporates reconstruction loss along with generalization error. \textit{However, this non-rigorous intuition is not precise. We will rigorously analyze both reconstruction and generalization later. } 

\begin{remark}[Regularization]
    In the non-parametric KRR setting, 
    % unlike in spectral clustering, performance depends on the full spectrum of the kernel matrix, not just the top eigenvectors. In particular, 
    $z>0$ helps to regularize the model by favoring eigenvectors that correspond to large eigenvalues; this is in the same spirit as PCA: when $d_i + z \gg d_i \PP^y_{u_i}$, the corresponding eigenvector will be negligible.
    This also resembles the idea of signal capture threshold from \cite{jacot2020}: if the signal along an eigenvector direction is too weak compared to the corresponding eigenvalue, the signal will not be learned by the model.  
\end{remark}

\begin{remark}\label{rm:reg}
When $z=0$, all eigenvectors contribute equally; however, when $z>0$, the eigenvectors corresponding to larger eigenvalues receive greater relative weights. Fine-tuning $z$ therefore initially seems to help informative eigenvectors stand out. However, as shown above, $z$ can be negligible in $d_i+z \gg d_i\PP_{u_i}^y$ if $d_i$ is large in this non-parametric regime.
% , $z$ is negligible. 
Conversely, $z$ has a non-negligible effect only when the labels are aligned with the trailing eigenvectors, indicating a less than ideal learning scenario. This finding is consistent with existing literature, such as \cite{liang2018,hastie2022}, suggesting that regularization is not essential to achieve good performance.
% Furthermore, since only a few eigenvectors corresponding to dominant eigenvalues are informative, it is natural that a label vector $y$ aligns better with those to be learned effectively.
% Additionally, this also coincides with recent findings on kernel alignments, such as \cite{canatar2020, jacot2020}, that 
\end{remark}

\subsection{Reconstruction Error}

Following our eigen-characterization of the KRR, we further analyze the reconstruction error.

\begin{proposition}\label{prop:mse}
    According to the non-parametric KRR assumptions, the reconstruction error of $y$, with respect to the kernel matrix $K_n$, is 
    \begin{equation}
    \err_{re} := \frac{1}{n}\|y - \hat{y}\|_n^2 =  \frac{1}{n} \sum_{i=1}^n \left(\frac{\PP_{u_i}^y\; z}{d_i+z}\right)^2\label{eq:loss}
    \end{equation}
    where $z \in \R^+$ is the regulation parameter, $\PP_{u_i}^y := u_i^\top y$ and $\{u_i\}_{i=1}^n$ are the eigenvectors of the kernel matrix $K_n$.
\end{proposition}
\begin{proof}

Following our previously established notations, we have 
\begin{align*}
    % R = \frac{1}{n}\tr((\hat{y}-y)^\top(\hat{y}-y)) &= \left|\sum_{i=1}^n \frac{(d_i \PP_{u_i}y)u_i}{d_i + z} - y\right|\\
   \hat{y}-y &= \sum_{i=1}^n \frac{(d_i \PP_{u_i}^y)u_i}{d_i + z} - y\\
    % &= \left( \sum_{i=1}^n \frac{d_i u_i u_i^\top}{d_i + z} - I_n \right) y\\
    &= \sum_{i=1}^n \frac{d_i u_i \PP_{u_i}^y}{d_i + z} - \sum_{i=1}^n u_i u_i^\top y\\
    &= \sum_{i=1}^n \frac{d_i u_i \PP_{u_i}^y}{d_i + z} - \sum_{i=1}^n u_i \PP_{u_i}^y\\
    % &= \sum_{i=1}^n \frac{d_i u_i \PP_{u_i}^y - u_i \PP_{u_i}^y(d_i + z)}{d_i + z}\\
    &= -\sum_{i=1}^n \frac{z u_i \PP_{u_i}^y}{d_i + z}\\
    % &= (U  \Lambda(\Lambda+ zI_n)^{-1} U^\top - I_n)y\\
    % &= U(\Lambda(\Lambda+ zI_n)^{-1} - I_n)U^\top y.
\end{align*}
Then, reconstruction error is
\begin{align*}
    \err_{re} &= \frac{1}{n}(\hat{y}-y)^\top(\hat{y}-y) \\
    &= \frac 1 n \left(\sum_{i=1}^n \frac{z u_i \PP_{u_i}^y}{d_i + z}\right)^\top\left(\sum_{j=1}^n \frac{z u_j \PP_{u_j}^y}{d_j + z}\right)\\
    % &= \frac{1}{n}\tr\left(U(\Lambda(\Lambda+ zI_n)^{-1} - I_n)U^\top y y^\top U(\Lambda(\Lambda+ zI_n)^{-1} - I_n)U^\top\right) \\
    % &= \frac{1}{n} \tr\left((\Lambda(\Lambda+ zI_n)^{-1} - I_n)U^\top y y^\top U(\Lambda(\Lambda+ zI_n)^{-1} - I_n)\right)\\
    &= \frac{1}{n} \sum_{i=1}^n \left(\frac{\PP_{u_i}^y\; z}{d_i+ z}\right)^2.
\end{align*}
\end{proof}
\begin{remark}
    Although it may be tempting to consider that the reconstruction error is minimal when $y \bot U$ (the eigenvector matrix), this is not valid. By construction, the columns of $U$ span the space of $\R^n \ni y$. In fact, for a given $y$, the quantity in the numerator, $\sum_{i=1}^n (\PP^y_{u_i})^2 = \|y\|^2$, is independent of the choice of the kernel. Hence, minimizing the reconstruction error is a matter of choosing a kernel that $d_i$ is large for whichever modes $u_i$ are most aligned with $y$ \cite{bordelon2020,simon2021}.%Hence, $y$ has to align with at least one of the eigenvectors and ideally to be aligned with the eigenvector corresponding to a large eigenvalue to minimize the mean squared error.
\end{remark}
According to \Cref{prop:mse}, for the KRR to have a small reconstruction error, the adjusted projection magnitude should be small compared to the corresponding adjusted eigenvalue (by the regularization parameter $z$). 
We observe that the reconstruction error can be minimized in the following sense.
\begin{proposition}[Reconstruction Error]\label{prop:recon}
    In non-parametric KRR, the reconstruction error can be decomposed as
    \begin{equation}
    % \| y - \hat{y}\|^2_n  
    \err_{re}= \sum_{i: d_i \neq 0} \frac{z^2 (\PP_{u_i}^y)^2}{(d_i + z)^2} + \sum_{i: d_i = 0}(\PP_{u_i}^y)^2. \label{eq:bvdecomp}
    \end{equation}
    If $\PP_{u_i}^y = 0$ for $i \in \{i : d_i = 0\}$, then for any $\varepsilon > 0$, the reconstruction error $\err_{re} < \varepsilon$ corresponds to some regularization parameter $z$.
\end{proposition}
\begin{proof}
From \cref{eq:loss}, we have    
\begin{align*}
\frac{1}{n} \sum_{i=1}^n \left(\frac{\PP_{u_i}^y\; z}{d_i+ z}\right)^2 
&= \frac{1}{n}\left( \sum_{i: d_i \neq 0} \left(\frac{\PP_{u_i}^y z}{d_i + z}\right)^2 + \sum_{i: d_i = 0}\left(\frac{\PP_{u_i}^y z}{d_i + z}\right)^2 \right)\\
&= \frac{1}{n}\left( \sum_{i: d_i \neq 0} \left(\frac{\PP_{u_i}^y z}{d_i + z}\right)^2 \right)
\end{align*}
The last equality holds due to our assumption $\PP_{u_i}^y = 0$ for $i \in \{i : d_i = 0\}$. Since 
$$
\lim_{z\rightarrow 0} \frac{1}{n}\left( \sum_{i: d_i \neq 0} \left(\frac{\PP_{u_i}^y z}{d_i + z}\right)^2 \right) = 0,
$$
the claim holds.
\end{proof}

\begin{remark}
    In theory, as long as the eigenvalue $d_i > 0$, we can find some $z$, such that $\forall \varepsilon >0, \err_{re} < \varepsilon$. However, in practice, this is limited by the machine precision or numerical error. 
\end{remark}
% 
% As the concept of dimension of learnability suggests, in order for KRR to reconstruct the learning target $y$, $y \in \mathcal{Y}_\varrho$ must hold. We will show next that $\mathcal{Y}_\varrho$ is actually a quite large space, as for any $y \in \mathcal{Y}_\varrho$, there is little assurance of generalization. 

\Cref{prop:recon} shows that given the learning targets, as long as the rank of the kernel matrix permits, the reconstruction error can be arbitrarily small by reducing the regularization parameters. This suggests a disconnect between the reconnection error and the generalization error, as illustrated below.
% To this end, and following our previous analysis, we can decompose the reconstruction error into two parts and $ \mathcal{Y}\varrho$ spanned by precisely the eigenvectors of the kernel matrix that correspond to non-zero eigenvalue. 

\begin{corollary}[Triviality of Reconstruction]\label{cor:triv}
    % The reconstruction $\hat{y}$ of a learning target $y$ using KRR with a full rank kernel matrix, 
    % can be decomposed with respect to its eigenvectors as follows:

    % where $z$ is the regularization parameter. Furthermore, the reconstruction error of the KRR depends only on the rank of the kernel matrix $K_n$ and the magnitude of the eigenvalues: the higher the eigenvalues, the lower the reconstruction error.  
    An arbitrary full-rank kernel matrix $K_n$ can reconstruct any learning target $y$ with $\err_{re} \approx 0$. Furthermore, the larger the eigenvalues, the smaller the reconstruction error $\err_{re}$. 
\end{corollary}
\begin{proof}
Due to \Cref{prop:recon}, we know that the full-rank kernel matrix can reconstruct any learning target $y$, since $d_i\neq 0, \forall i\in \{i: \; \PP_{u_i}^y \neq 0\}$.
It is also clear that larger $d_i$ will lead to smaller $\err_{re}$ based on \cref{eq:loss}.

\end{proof}

Before we attribute the above phenomenon to overfitting, we can make the following observations: 1) the inaccuracy in estimating the eigenvectors and eigenvalues from the sample does not necessarily result in poor reconstruction error;
2) as long as $d_i \gg z$, the corresponding contribution to the reconstruction error will remain near zero;
3) the eigencomponents corresponding to zero eigenvalues cannot be reconstructed.
% It is readily seen that if the target lives in the span of top eigenvectors, the variance terms in \cref{eq:bvdecomp} vanish. 
This demonstrates that reconstruction in non-parametric KRR is relatively a trivial task; in some loose sense this resembles the 1-nearest neighbor predictor. As long as the kernel matrix is of high rank, we can achieve near-zero reconstruction (training) error. The error is primarily a limitation imposed by the numerical stability associated with the use of small $z$. The triviality of reconstruction arises from the perfect alignment of the training and predicting kernels in the reconstruction process. It is worth mentioning that this is purely an algebraic fact, independent of data distributions. 

We will also demonstrate empirically in \Cref{sub:overfit},
% $\err_{re} \approx 0$
% The above corollary also suggests that,
in the non-parametric KRR that the near-zero reconstruction error is not necessarily indicative of overfitting. The actual quality of generalization depends on more nuanced factors, which will be analyzed in the following section. 

\subsection{Generalization Error and Predictive Eigenalignments}

The previous section showed that the training or reconstruction error in non-parametric KRR is connected to eigenvalues and eigenvectors of the sample kernel $K_n$, as well as the alignment of these eigenvectors with the training targets. Next, we will demonstrate how the generalization error depends on how \emph{accurately} those sample eigenvalues and eigenvectors are estimated from a given finite data set.

For a finite sample size, not all eigenvectors and eigenvalues can be estimated equally well. It is known that with a finite sample size, the top eigenvalues and eigenvectors can be approximated more accurately in terms of relative error\footnote{The actual dynamics of estimation accuracy are more nuanced, but it is outside the scope of this paper.} (see also \eg \cite{mestre2008,loukas2017}). This suggests that there are two distinct sources of error: 1) a kind of ``bias'' due to the inability to represent the target solely by the informative eigenvectors (those with large eigenvalues, which are more easily estimated from finite data); and 2) a kind of ``variance'' due to the necessity of representing the target with the non-informative eigenvectors (those with small eigenvalues, which are not easily (or feasible) estimated from finite data which will cause ``confusion'' in predictions). 
% \textit{In contrast to the typical theory of the trade-off of bias and variance, there is no apparent trade-off in our interpretation: When the target is fixed, the decomposition is also fixed.}
\subsubsection{Eigen-characterization of generalization}

Generalization is the question of how well any specific data sample $\D$ of size $n$ can approximate other data samples $\tilde{\D}$ drawn from the same distribution. We derive a novel expression for the generalization error in terms of the relative errors of the eigenvalue and eigenvector estimates. 
Intuitively, we notice that the generalization behavior of a non-parametric Kernel Ridge Regression (KRR) model depends on the extent to which the kernel matrix is affected by the samples used in the training process. If the predictions are highly dependent on volatile information, such as trailing eigenvectors, then the generalization error is likely to be large. To this end, we utilize existing tools in matrix analysis to examine how the different samples affect the eigenvectors. 
\textit{We first illustrate this idea non-rigorously.}
% Analogous to our treatment of reconstruction error above, we will use the sigular value decomposition of the $n_{test} \times n$ predictive kernel matrix $\tilde{K} := \tilde{V}\tilde{\Lambda}\tilde{U}^\top$ in our analysis: 
% \begin{align*}
%     \hat{\tilde{y}} &= \tilde{V}\tilde{\Lambda}\tilde{U}^\top \sum_{i=1}^n \frac{u_iu_i^\top y}{d_i + z} \\
%     &= \sum_{i=1}^n \frac{\PP_{u_i}^y(\tilde{V}\tilde{\Lambda}\tilde{U}^\top u_i)}{d_i + z} \\
%     &= \sum_{i=1}^n \frac{\PP_{u_i}^y(\sum_{j=1}^{n_{test}}\lambda_j\tilde{v}_j\tilde{u}_j^\top u_i)}{d_i + z} \\
%     &= \sum_{i=1}^n \frac{\PP_{u_i}^y(\sum_{j=1}^{n_{test}}\lambda_j\PP_{\tilde{u}_j}^{u_i}\tilde{v}_j)}{d_i + z} \\
%     &= \sum_{j=1}^{n_{test}} \lambda_j \tilde{v}_j \sum_{i=1}^n\frac{\PP^y_{u_i}\PP^{u_i}_{\tilde{u}_j}}{d_i + z} \num \label{eq:testpred}
% \end{align*}

% In test predictions, the predictions are weighted left singular vectors ($\tilde{v}_j$) where the weights are jointly determined by the alignment between $y$ and the eigenvectors $u_i$ ($\PP_{u_i}^y$), and the alignments between eigenvectors $u_i$ and right singular vectors $\tilde{u}_i$ ($\PP^{u_i}_{\tilde{u}_j}$). In the case where $\tilde{K} = K$ (i.e. where test data equals training data), \cref{eq:testpred} simplifies to the reconstruction formula \cref{eq:recon}.
For each individual test point $\tx_j$, we can write the prediction as
$$
\hat{\tilde{y}}_j = \sum_{i=1}^n\frac{\PP^y_{u_i}\PP^{u_i}_{\hkkx}}{d_i + z},
$$
where $\hkkx$ is the $j$th row of $\tilde{K}$ corresponding to the new data point $\tx_j$. To better understand predictive behavior, it is worth exploring how predictions for generalization differ from those for reconstruction.

To this end, we can assume that the vector $\hkkx$ comes from a row of the kernel matrix $\tilde{K}$ with the same training size, \ie $n\times n$ (this allows spectral decomposition as it pertains to the training kernel matrix). We also temporarily overlook the fact that the test example $\hkkx$ exists solely in the test kernel $\tilde{K}$, and we allow decomposition $\tilde{K} = \sum_{i=1}^n \tilde{d}_i \tilde{u}_i \tilde{u}_i^\top$. We have
\begin{align*}
    \hat{\tilde{y}}_j &= \sum_{i=1}^n\frac{\PP^y_{u_i}\PP^{u_i}_{\hkkx}}{d_i + z}\\
   &= \sum_{i=1}^n \frac{\PP^y_{u_i} [\sum_k^n \tilde{d}_k \tilde{u}_k \tilde{u}_k^\top u_i]_j }{d_i + z} \\
   &= \sum_{i=1}^n \frac{\PP^y_{u_i} [\tilde{d}_i\tilde{u}_i\tilde{u}_i^\top u_i+\sum_{k \neq i} \tilde{d}_k \tilde{u}_k \tilde{u}_k^\top u_i]_j }{d_i + z} \\
   % &= \sum_{i=1}^n \frac{\PP^y_{u_i} \tilde{d}_i \tilde{u}_i^{(j)}  }{d_i + z}\\
\end{align*}
where $[\cdot]_j$ denotes the selection of the $j$th element of the vector. We can then make several observations: 1) the sample size of the training set determines the size of the hypothetical kernel matrix for the testing data. This is consistent with the fact that the size of the testing data should not affect the test accuracy; 2) Small $\tilde{u}_k^\top u_i$ for $i \neq k$ and large $\tilde{u}_k^\top u_i$ for $i = k$ will lead to a correct prediction, as both are approaching their respective optimal counterparts.
\subsubsection{Generalization Bounds}
To analyze the generalization error, we restrict ourselves to considering the target function $f \in \hil$, where $\hil$ represents the Reproducing Kernel Hilbert Space (RKHS) corresponding to the Mercer kernel $\kfn$. 
Following the standard convention (see \eg \citet{cucker2007}), we can write
$\kfn(x_j, x_k) = \sum_i \sqrt{d_i} \psi_i(x_j)\psi_i(x_k)$, where $\{d_i, \psi_i\}$ represent the eigenvalues and eigenfunctions of a linear operator induced by $\kfn$. Since scaled eigenfunctions $\{\sqrt{d_i}\psi_i\}$ form an orthonormal basis in $\hil$, we can express $f = \sum_i a_i\psi_i$.
For any finite training sample of size $n$, we have a predictive function $\hat{f}_n \in \hil$ obtained via kernel ridge regression. By the representer theorem, $\hat{f}_n = \sum_i^n \alpha_i \kfn(x_i, \cdot)$, where $\alpha = (K_n + zI)^{-1}y \in \R^n$.
% With a finite sample training set, under the previously utilized setup, we have $\hat{f}_n = (K_n + zI_n)^{-1}y \in \hil_n$. Although, technically, $\hat{f}_n$ is not the actual predictive function used to make predictions on new data points, it uniquely determines one. 
% This formulation enables us to effectively derive the risk bound. 
As we increase the sample size $n$, we observe
$$
\lim_{n\rightarrow \infty} \hat{f}_n = f,
$$
where $f \in \hil$ corresponds to the ground truth generalization error minimizer.
Then, for the new test point $\tx_j$, let $\hat{\tilde{\kk}}_{\tx_j}$ be the row vector (of length $n$) of the kernel matrix $\tilde{K}$ that corresponds to $\tx_j$. Additionally, $\tilde{\kk}_{\tx_j} \in \hil$ is the kernel map that corresponds to $\tilde{x}_j$. 
% With a slight abuse of notation for $f$ and $\hat{f}_n$, 
We can express the prediction for $\tx_j$ as

$$
\ty_j = f(\tx_j) := \langle \tilde{\kk}_{\tx_j}, f\rangle_{\hil},\text{ and } \hat{\ty}_j = \hat{f}_n(\tx_j) := \langle \hat{\tilde{\kk}}_{\tx_j}, (K_n + zI_n)^{-1}y\rangle \, .
$$
Moreover, the predictive error is
$$
\hat{\tilde{y}}_j - \tilde{y}_j = \langle \hat{\tilde{\kk}}_{\tx_j}, (K_n + zI_n)^{-1}y\rangle - \langle \tilde{\kk}_{\tx_j}, f\rangle_{\hil}.\\
$$
% The convergence relies on both the convergence of $\tilde{\kk}_{\tx_j}$ and the convergence of $(K_n + z I)^{-1}y$. 
The accuracy of the prediction $\hat{y}_j = \hat{f}_n(\tilde{x}_j)$ depends on the empirical estimator $\hat{f}_n \in \mathbb{H}$—whose coefficients $\boldsymbol{\alpha} = (K_n + zI_n)^{-1}y \in \mathbb{R}^n$ are derived via the Representer Theorem—providing a tight approximation of the true target function $f \in \mathbb{H}$ with respect to the RKHS norm. Consequently, this convergence requires that the empirical kernel operators constructed from the finite sample evaluations (represented by $K_n$ and $\mathbf{k}_{\tilde{x}_j}$) sufficiently concentrate around the true integral operator and covariance structures of the underlying data-generating distribution as $n \to \infty$.
% The accuracy of the prediction $\tilde{y}_j$ depends on $\hat{f}_n = (K_n + zI_n)^{-1}y \in \hil$ providing an accurate approximation of $f$, as well as $\hat{\tilde{\kk}}_{\tx_j}$ (the finite sample approximation), which must also yield an accurate approximation of the ground truth $\tilde{\kk}_{\tx_j}$. 
% Let $\{\lambda_i, \psi_i\}_i$ be the eigenvalues and eigenfunctions of the kernel $K$, assume that the ground truth $f \in \hil$ can be expressed as $f = \sum_i a_i\psi_i$. 

Our key idea is to compare the learned predictor $\hat{f}_n$ not directly to the ground-truth function $f$, but rather to a hypothetical best predictor, $f^*_n$ which is learned from only $n$ samples, where the samples themselves are selected to minimize the generalization error of $f^*_n$. That is, we define $f_n^*$ as the KRR solution to the hypothetical optimal data sample $\{x^*_1, \ldots, x^*_n\}$, with the values of $x^*$ selected to minimize generalization error. Note that this will require the order of the optimal data points to be ``aligned'' with the order of the actual data sample in a sense that we will make precise below.
%Accordingly, our analysis proceeds under the following assumption.
% \begin{assumption}[Well-represented Samples]
%     The eigenspace of the empirical kernel matrix $K_n$ is structurally aligned with that of the optimal kernel matrix $K_n^*$, such that $K_n$ can be modeled as a bounded symmetric perturbation of $K_n^*$ possessing a stable spectral gap.
% \end{assumption}
% The above assumption appears to be strong; however, it is reasonable to assume in practice when we have a sufficiently large dataset. We argue that it is infeasible to predict the generalization errors when the dataset we have is not representative of the population. 

% For a given sample size $n$, there is an optimal finite approximation $\hil_n^* \ni f^*_n = \sum_{i=1}^n \frac{\PP_{u^*_i}^y}{d^*_i + z}u^*_i$ of $f$, which can be expressed as the KRR solution given $n$ hypothetical optimal data points. 
Let $K^*_n$ be the kernel matrix for these $n$ optimal data points, and let $\{d^*_i, u^*_i\}$ be the eigenvalues and eigenvectors of $K_n^*$. The optimal $\tilde{\kk}_{\tx_j}^*$ and the predictive $\hat{\tilde{\kk}}_{\tx_j}$ (utilizing an actual non-optimal training set) also have their representations with corresponding eigenvectors: 
$$
\tilde{\kk}_{\tx_j}^* = \sum_{i=1}^n \beta^*_i u^*_i \text{, and } \hat{\tilde{\kk}}_{\tx_j} = \sum_{i=1}^n \hat{\beta}_i u_i,
$$
recall that $\{d_i, u_i\}$ are the eigenpairs of the actual training kernel matrix $K_n$.
Since $\hat{\tilde{\kk}}_{\tx_j}$ depends on the arbitrary input $\tx_j$, which determines $\hat{\beta}_i$, an accurate value of $\hat{\beta}_i$ will result in an accurate prediction for $\tx_j$. To this end, we have
\begin{align*}
|\hat{\tilde{y}}_j - \tilde{y}_j | 
&= |\langle \hat{\tilde{\kk}}_{\tx_j}, (K_n + zI_n)^{-1}y\rangle - \langle \tilde{\kk}_{\tx_j}, f\rangle_{\hil}|\\
&= \left|\left(\langle \hat{\tilde{\kk}}_{\tx_j}, (K_n + zI_n)^{-1}y\rangle - \langle \tilde{\kk}^*_{\tx_j}, f^*_n\rangle\right) + \left(\langle \tilde{\kk}^*_{\tx_j}, f^*_n\rangle - \langle \tilde{\kk}_{\tx_j}, f\rangle_{\hil} \right)\right|\\
&= \left| \left( \left\langle \sum_{i=1}^n \hat{\beta}_iu_i, \sum_{i=1}^n \frac{\PP_{u_i}^y}{d_i + z} u_i \right\rangle - \left\langle \sum_{i=1}^n \beta^*_iu^*_i, \sum_{i=1}^n \frac{\PP_{u^*_i}^y}{d^*_i + z} u^*_i \right\rangle \right) + \left(\langle \tilde{\kk}^*_{\tx_j}, f^*_n\rangle - \langle \tilde{\kk}_{\tx_j}, f\rangle_{\hil} \right)\right|\\
&\leq \left| \sum_{i=1}^n \hat{\beta}_i \frac{\PP_{u_i}^y}{d_i + z} - \sum_{i=1}^n \frac{\PP_{u^*_i}^y}{d^*_i + z}\beta^*_i \right| + \left|\langle \tilde{\kk}^*_{\tx_j}, f^*_n\rangle - \langle \tilde{\kk}_{\tx_j}, f\rangle_{\hil} \right| \\
&\leq \underbrace{ \left| \sum_{i=1}^n \hat\beta_i \left( \frac{\PP_{u_i}^y}{d_i + z} - \frac{\PP_{u^*_i}^y}{d^*_i + z}\right) \right| }_{\text{Term} 1}
+ \underbrace{\left|\sum_{i=1}^n \frac{\PP_{u^*_i}^y}{d^*_i + z}\Delta\beta_i\right| }_{\text{Term} 2}
+ \underbrace{\left|\langle \tilde{\kk}^*_{\tx_j}, f^*_n\rangle - \langle \tilde{\kk}_{\tx_j}, f\rangle_{\hil} \right|}_{\text{Term} 3}, \num\label{pf:1}
\end{align*}
where we assume the orthonormality of the eigenvectors and define $\Delta \beta_i = \hat\beta_i - \beta^*_i$. %We observe that $|\ty_j - \hat{\ty}_j|$ is minimized when $\PP^y_{u_i}$ is near zero, except for the top few eigenvectors. 
Observing \cref{pf:1}, ``Term 1'' captures the error caused by incorrect eigenvalue estimation; ``Term 2'' captures the error caused by the bias of the training set; and ``Term 3'' captures the error caused by the finite sample size.

Bounding the generalization error amounts to bounding the first two terms in \cref{pf:1}, as the last term is fixed. To this end, we will consider the finite sample kernel matrix of the training dataset as a perturbed version of the optimal sample kernel matrix since our performance depends on the finite estimation of the eigenfunctions. We can then utilize the tools from matrix perturbation theory to develop our bound on \cref{pf:1}.

We first introduce a well-known lemma that will aid us in developing our results.
\begin{lemma}[Davis-Kahan \cite{davis1970, eldridge2017}]\label{lemma:davis}
    % Suppose that the perturbed matrix $\tilde{X}= X + \Delta X$, where $X$ and $\Delta X$ are symmetric matrices, let $\delta_t = \min\{|\tilde{\lambda}_j - \lambda_t|: j \neq t\}$, then 
    % $$
    % \sin \theta_t \leq \frac{\|\Delta X\|}{\delta_t},
    % $$
    % where $\Delta X$ is the perturbation error, and $\theta_t = \cos^{-1}\left( |u_t^\top \tilde{u}_t| \right)$.
   Suppose $X$ and $\Delta X$ are symmetric matrices, and let the perturbed matrix be defined as
    $\tilde{X} = X + \Delta X$. Let $(\lambda_i, u_i)$ and $(\tilde{\lambda}_j, \tilde{u}_j)$ denote the eigenvalue-eigenvector pairs of $X$ and $\tilde{X}$, respectively. If we define the spectral gap for a given index $t$ as $\delta_t = \min_{j \neq t} \vert{}\tilde{\lambda}_j - \lambda_t\vert{}$, then the principal angle $\theta_t$ between the unperturbed eigenvector $u_t$ and the perturbed eigenvector $\tilde{u}_t$, defined by $\theta_t = \cos^{-1}(\vert{}u_t^\top \tilde{u}_t\vert{})$, satisfies the inequality:
    $$
    \sin \theta_t \leq \frac{\|\Delta X\|}{\delta_t}
    $$
    where $\|\Delta X\|$ denotes the operator norm of the perturbation matrix $\Delta X$.
\end{lemma}
It is clear that $\delta_t$ is large for isolated eigenvalues and small for trailing eigenvalues. While there are tighter bounds than \Cref{lemma:davis}, such as \cite{eldridge2017, fan2018}, these bounds are generally more involved in their formulation; for the brevity of the result, we focus on \Cref{lemma:davis} in our paper.

We next demonstrate that, given a fixed sample kernel error $\Delta K_n = K_n-K^*_n$, aligning learning targets with the top eigenvectors will result in improved generalization performance. 
% Since we are interested in the finite estimation of eigenfunctions, we assume that the sample kernel error is calculated with aligned eigenvectors, such that
%     $$
%     \|\Delta K_n\| = \min_Q\|  K_n  - QK_n^*Q^\top\|,
%     $$
% where $Q$ is the permutation matrix, and the norm is to be interpreted as the operator norm, which is the largest eigenvalue of a symmetric matrix. Here, we assume that the optimal kernel is constructed with similar structure, in terms of the observation order. \footnote{Admittedly, this is a strong assumption. However, our analysis can still show how much perturbation from the kernel matrix can contribute to the generalization error. } This is to say, with sufficiently large samples, our samples can roughly capture similar structures to that of the data. 
% Alternatively, we consider our training set to be a perturbed version of an optimal training sample of a given size. 
Our analysis aims to demonstrate how the generalization performance is affected by the perturbations. 
% \textcolor{red}{[TODO] as defined, there is no guarantee that $\Delta K_n$ is small, since there were no constraints on the ordering of optimal data points. The solutions $\hat f$ and $f^*$ are of course insensitive to the data order, but the kernels themselves are not!}, 
To this end, for a fixed kernel difference $\Delta K_n$, we have
\begin{equation}\label{eq:Theta}
|u_i^\top u^*_i| \geq \sqrt{\max\left\{0,1- \left(\frac{\|\Delta K_n\|}{\delta_i}\right)^2\right\}} =: \Theta_i
\end{equation}
since $u_i$ and $u^*_i$ are unit vectors, $\delta_i = \min\{|d_k - d^*_i|: i \neq k\}$.

% Note that, up until this point, we have not explicitly ``aligned'' the matrices $K_n$ and $K^*_n$. While KRR predictions are invariant to permutations of the data points, misalignment between $K_n$ and $K^*_n$ makes the magnitude of $\Vert \Delta K_n \Vert$ large. However, this is not a problem for our analysis; the bound in Lemma \ref{lemma:davis} holds, albeit loosely, for large $\Vert \Delta X \Vert$. In what follows, this implies that we have a collection of $n!$ upper-bounds on generalization performance, one for each permutation of the data points, all of which are valid but most of which are not tight. Among all permutations of the optimal data sample $\{x^*_1, \ldots, x^*_n\}$, one may take $K^*_n$ to be defined by the permutation that minimizes $\Vert \Delta K_n \Vert$, giving the tightest bound.

\subsubsection*{The upper bound of term 1 in \cref{pf:1}}

To bound term 1, 
% we first notice that we can write $\frac{\PP_{u^*_i}^y}{d^*_i + z} = \PP_{u^*_i}^y/(d^*_i + z)$, and both $\PP_{u_i}^y/(d_i + z)$ and $\frac{\PP_{u^*_i}^y}{d^*_i + z}$ are empirical estimations of $a_i$. 
suppose $d_i = d^*_i + \Delta d_i$,  we have
\begin{align*}
% \frac{\PP_{u_i}^y}{d_i + z} - \frac{\PP_{u^*_i}^y}{d^*_i + z} 
    % &= 
\frac{\PP_{u_i}^y}{d_i + z} - \frac{\PP_{u^*_i}^y}{d^*_i + z} 
    &= \frac{\PP^y_{u_i} (d^*_i + z) - \PP^y_{u^*_i}(d_i+z)}{(d_i + z)(d^*_i + z)} \\
    % &= \frac{\PP^y_{u_i} d^*_i - \PP^y_{u^*_i}d_i  + z(\PP^y_{u_i} - \PP^y_{u^*_i})}{(d_i + z)(d^*_i + z)} \\
    &= \frac{(d^*_i + z)( \PP^y_{u_i} - \PP^y_{u^*_i}) - \Delta d_i \PP_{u^*_i}^y}{(d_i + z)(d^*_i + z)} \\
    &= \frac{(d^*_i + z)\langle{u_i - u^*_i, y}\rangle - \Delta d_i \PP_{u^*_i}^y}{(d_i + z)(d^*_i + z)} \, . \num\label{eq:interm}\\
\end{align*}
To gain a sense of the concentration behavior of the above expression. We make use of the following lemma, which is modified from \cite{loukas2017}. 
\begin{lemma}[Theorem 3.2 of \cite{loukas2017}]\label{lemma:loukas}
    For Hermitian matrices $K$ and $\tilde{K} = \Delta K + K$, with eigenvectors $u_i$ and $\tilde{u}_j$ respectively, we have
    $$
    |\langle u_i, \tilde{u}_j\rangle| \leq \frac{2\|\Delta K\|}{|d_i - d_j|}
    $$
    provided $d_i \neq d_j$, and the perturbation is sufficiently small such that the condition
    $$
    \sign(d_j - d_i)2\tilde{d}_i > \sign(d_j - d_i)(d_i + d_j)
    $$ holds true.
\end{lemma}
\begin{remark}
    Compared to the original Theorem 3.2 in \cite{loukas2017}, we have adapted the \Cref{lemma:loukas} with our notations and adjusted the right-hand side based on the simple fact that $\|\Delta K_n u_i\|_2 \leq \|\Delta K_n\|$.
\end{remark}
   % For any two eigenvectors $u_i$ and $\tilde{u}_j$ of kernel matrices from different data samples $X$ and $\tilde{X}$, respectively. Their eigenvalues are $\{\lambda_i\}_i$ and $\{\tilde{\lambda}_i\}_i$, respectively. For $\lambda_i \neq \lambda_j$ and any $\varepsilon > 0$, we have
   % \begin{equation}
   %      \PP(|\langle u_i, \tilde{u}_j \rangle| \geq \varepsilon) \leq \frac{4\E(\|\Delta K_n \|^2)}{\varepsilon^2(\lambda_i - \lambda_j)^2}
   % \end{equation}
   % holds for $2\sign(\lambda_i - \lambda_j)  \tilde{\lambda}_i > \sign(\lambda_i - \lambda_j)(\lambda_i + \lambda_j)$, where $\Delta K_n$ denotes the difference between two kernel matrices.  
   
% \end{lemma}
% \begin{proof}
%     We follow the same arguments as \cite{loukas2017}.  
%     % Consider  empirical kernel 
%     % $$
%     % \tilde{K} = \frac{1}{n}\sum_{i=1}^n \{k(x_{i,p}, x_{i,q})\}_{p, q=1}^m,
%     % $$
%     % where  
%     For eigenvectors $u_i$ and $\tilde{u}_j$ of K and $K + \Delta K_n$, respectively:
%     \begin{align*}
%         \PP(|\langle u_i, \tilde{u}_j\rangle| \geq \varepsilon) &= \PP(\langle u_i, \tilde{u}_j\rangle^2 \geq \varepsilon^2) \\
%         &\leq \frac{\E(\langle u_i, \tilde{u}_j\rangle^2)}{\varepsilon^2} \\
%         &\leq \frac{4\E(\|\Delta K_n \|^2)}{\varepsilon^2(\lambda_i - \lambda_j)^2}
%     \end{align*}
% where first equality is due to Markov's inequality, and for last inequality, according to \cite{loukas2017}, we have
% $$
% \E (\langle u_i,\tilde{u}_j\rangle^2) \leq \frac{4\E (\| \Delta K_n\|^2)}{(\lambda_i - \lambda_j)^2}.
% $$
% \end{proof}

Lemmas \ref{lemma:davis} and \ref{lemma:loukas} both contain bounds in terms of $\Delta K_n$, and Lemma \ref{lemma:loukas} in particular requires that $K$ and $\tilde K$ (which will be $K_n$ and $K_n^*$ in what follows) are ``close.'' There are two ways in which $\Delta K_n = K_n - K^*_n$ may be large. 
First, the optimal $\{x^*_1, \ldots, x^*_n\}$ may not be unique. Second, the \emph{order} of training data must be aligned with the order of the optimal data, i.e. $x_i$ should be close to $x^*_i$ for all $i$. We do not specify $K^*_n$ directly; instead, we assume that, among all possible permutations of data and among all possible sets of optimal data points, we take $K^*_n$ to be whichever minimizes misalignment to $K_n$. We do not think this is a strong assumption; the condition $\sign(d_j - d_i)2\tilde{d}_i > \sign(d_j - d_i)(d_i + d_j)$ in  \Cref{lemma:loukas} is known to hold asymptotically almost surely for large eigengaps and sufficiently large sample sizes~\cite{loukas2017}. \textit{We will assume this in our further analysis, as it further supports our claim that when the learning targets align with the trailing eigenvectors, errors cannot be bounded.}

To bound \cref{eq:interm}, we first observe that through straightforward algebra, we have
\begin{align*}
    u_i - u^*_i &= \sum_k^n \langle u_i, u^*_k\rangle u^*_k - u^*_i  \\
    &= \sum_{k: i \neq k} \langle u_i, u^*_k\rangle u^*_k + \langle u_i, u^*_i\rangle u^*_i - u^*_i, \num\label{eq:uu}
    % &\leq \sum_{k: i \neq k} \varepsilon_k u^*_k + \langle u_i, u^*_i\rangle u^*_i - u^*_i.
\end{align*}
additionally, we have
\begin{align*}
    \left|\hat\beta_i\left(\frac{\PP_{u_i}^y}{d_i + z} - \frac{\PP_{u^*_i}^y}{d^*_i + z}\right)\right| 
    &\leq |\hat\beta_i|  \left|   \frac{(d^*_i + z)( \PP^y_{u_i} - \PP^y_{u^*_i}) - \Delta d_i \PP_{u^*_i}^y}{(d_i + z)(d^*_i + z)}  \right|\\
    &= |b_i| (d_i + z)  \left| \frac{(d^*_i + z)( \PP^y_{u_i} - \PP^y_{u^*_i}) - \Delta d_i \PP_{u^*_i}^y}{(d_i + z)(d^*_i + z)} \right|\\
    &= |b_i| \left|  \frac{(d^*_i + z)( \PP^y_{u_i} - \PP^y_{u^*_i}) - \Delta d_i \PP_{u^*_i}^y}{d^*_i + z} \right|\\
    &= |b_i| \left| ((u_i^\top - u_i^{*\top})y) - \frac{\Delta d_i \PP_{u^*_i}^y}{d^*_i + z} \right|\\
    &\leq |b_i| \left( |(u_i^\top - u_i^{*\top})y| + \left|\frac{\Delta d_i \PP_{u^*_i}^y}{d^*_i + z} \right|\right),\\
\end{align*}
where $b_i := \hat{\beta}_i / (d_i+z)$. Plugging \cref{eq:uu} into the last line, we obtain
\begin{align*}
    \left|\hat\beta_i\left(\frac{\PP_{u_i}^y}{d_i + z} - \frac{\PP_{u^*_i}^y}{d^*_i + z}\right)\right| 
    &\leq |b_i| \left( \left|\left(\sum_{k: i \neq k} \langle u_i, u^*_k\rangle u^*_k + \langle u_i, u^*_i\rangle u^*_i - u^*_i\right)^\top y\right| + \left|\frac{\Delta d_i \PP_{u^*_i}^y}{d^*_i + z} \right|\right).\\
\end{align*}
% Since the only terms that matter are those with $\PP_{u^*_k}^y \neq 0$, Using \Cref{lemma:loukas}, with probability \mbox{$\prod_{k\neq i: \PP_{u^*_k}^y \neq 0}1-\frac{4\E(\|\Delta K_n \|^2)}{\varepsilon_k^2(d^*_i - d^*_k)^2}$}, we can further bound the above equations as follows:
Using \Cref{lemma:loukas}, we have
\begin{align*}
    \left|\hat\beta_i\left(\frac{\PP_{u_i}^y}{d_i + z} - \frac{\PP_{u^*_i}^y}{d^*_i + z}\right)\right| 
    &\leq  |b_i|\left(\left| \sum_{k:k\neq i} \frac{2\|\Delta K_n\|}{|d^*_i-d^*_k|} \PP_{u^*_k}^y\right| + \left| \langle u_i, u^*_i\rangle u_i^{*\top}y - u_i^{*\top}y \right| + \left| \frac{\Delta d_i \PP_{u^*_i}^y}{d^*_i + z} \right|\right)\\
    &\leq  |b_i|\left(\left|  \sum_{k:k\neq i} \frac{2\|\Delta K_n\|}{|d^*_i-d^*_k|} \PP_{u^*_k}^y\right| + \left| \PP_{u^*_i}^y (1 - \Theta_i) \right| + \left| \frac{\|\Delta K_n \| \PP_{u^*_i}^y}{d^*_i + z} \right| \right), \num\label{pf:2}\\
\end{align*}
where the last inequality is due to $|\Delta d_i| \leq \| \Delta K_n \|$, which follows from the definition of the operator norm, and \cref{eq:Theta}. Notice that the summation term contributes only when $\PP_{u_K^*}^y \neq 0$.

\subsubsection*{The upper bound of term 2 in \cref{pf:1}}

Next, we bound the term 2 on the right-hand side of \cref{pf:1}.
Since both $\tilde{\kk}_{\tx_j}^*$ and $\hat{\tilde{\kk}}_{\tx_j}$ are the $n$-dimensional projections of the infinite-dimensional counterpart $\tilde{\kk}_{\tx_j}$, we also have
\begin{align*}
    \left|\sum_{i=1}^n \frac{\PP_{u^*_i}^y}{d^*_i + z} \Delta\beta_i\right| &=\left|\sum_{i=1}^n \frac{\PP_{u^*_i}^y}{d^*_i + z} (\hat\beta_i - \beta^*_i)\right|\\
    &= \left|\sum_{i=1}^n \frac{\PP_{u^*_i}^y}{d^*_i + z} \left( \PP_{u_i}^{\hkkx} - \PP_{u^*_i}^{\kkx^*} \right)\right|\\
    &\leq \sum_{i=1}^n \frac{ |\PP_{u^*_i}^y| }{d^*_i + z} \|\kk\| \left(2+\sqrt{2- 2u_i^\top u^*_i}\right) . \num\label{pf:3}
\end{align*}
where the last inequality is due to Cauchy-Schwarz, and $\| \kk\| := \max(\|\hkkx\|, \|\kkx^*\|)$. We summarize the analysis presented above in the following theorem.

\begin{theorem}\label{thm:gen}
    Using previously established notations and within the framework of non-parametric Kernel Ridge Regression (KRR), consider a training set of size $n$ with the kernel matrix represented as $K_n$. define
    \begin{align*}
    \mathfrak{S}_j :=
      C_j\sum_{i=1}^n|\PP_{u^*_i}^y|\left(\sum_{k:k\neq i} \frac{2\|\Delta K_n\|}{|d^*_i-d^*_k|}  +  (1 - \Theta_i) +  \frac{\|\Delta K_n\| + 2+\sqrt{2- 2\Theta_i} }{d^*_i + z} \right) \num\label{eq:bound}\\
    \end{align*}   
    where $C_j$ is some constant dependent on $\tx_j$, and $\Delta K_n = K_n - K^*_n$ and $K^*_n$ is the optimal kernel matrix of size $n$, $\{d_i, u_i\}$ and $\{d^*_i, u^*_i\}$ are the eigenpairs of $K_n$ and $K_n^*$, respectively,  
    and $\Theta_i$ is defined in \cref{eq:Theta}. Assuming that $2\sign(d^*_i - d^*_k)  d_i > \sign(d^*_i - d^*_k)(d^*_i + d^*_k)$ holds, the expected squared generalization error
    % with probability \mbox{$\prod_{i: \PP_{u^*_i}^y \neq 0}1-\max_{k: k \neq i, \PP_{u^*_k}^y \neq 0}\frac{4\E(\|\Delta K_n \|^2)}{\varepsilon^2(d^*_i - d^*_k)^2}$}, 
    can be bounded as
    % such that
    \begin{equation}
    \err_{ge} := \E (\ty_j - \hat{\ty}_j)^2 \leq \E\left(\mathfrak{S}_j + \left|\langle \tilde{\kk}^*_{\tx_j}, f^*_n\rangle - \langle \tilde{\kk}_{\tx_j}, f\rangle_{\hil} \right| \right)^2. \label{eq:errbound}
    \end{equation}
\end{theorem}
\begin{proof}
Let the estimation error
$$
\err_{es}(\tx_j) := |\ty_j - \hat{\ty}_j| + \left|\langle \tilde{\kk}^*_{\tx_j}, f^*_n\rangle - \langle \tilde{\kk}_{\tx_j}, f\rangle_{\hil} \right|. 
$$
We then combine \cref{pf:1}, \cref{pf:2}, and \cref{pf:3} to obtain (note that $\hat\beta_i$ and $\beta^*_i$ are dependent on $\tx_j$).
\begin{align*}
    \err_{es}(\tx_j) &\leq \left| \sum_{i=1}^n  \hat\beta_i \left( \frac{\PP_{u_i}^y}{d_i + z} - \frac{\PP_{u^*_i}^y}{d^*_i + z}\right) \right| + \left|\sum_i \frac{\PP_{u^*_i}^y}{d^*_i + z}\Delta\beta_i\right| \\
    % &\leq \sum_{i=1}^n |b_i| \left(\|y\|\sqrt{2-2\Theta_i} + \frac{|\PP_{u^*_i}^y|\|\Delta K_n\|}{|\hat{\lambda}_i| +z}\right) + \left|\sum_{i=1}^n \hat a_i \left( \|\kkx\| \sqrt{2- 2u_i^\top u^*_i} \right)\right|  \\
    &\leq \sum_{i=1}^n |b_i| \left((u_i^\top - u_i^{*\top})y + \frac{|\PP_{u^*_i}^y|\|\Delta K_n\|}{d_i^* +z}\right) + \sum_{i=1}^n \frac{|\PP_{u^*_i}^y|}{d^*_i + z} \|\kk\|  \left(2+\sqrt{2- 2\Theta_i} \right)  \\
    &\leq \sum_{i=1}^n|b_i|\left(\left|  \sum_{k:k\neq i} \frac{2\|\Delta K_n\|}{|d^*_i-d^*_k|} \PP_{u^*_k}^y \right| + \left| \PP_{u^*_i}^y (1 - \Theta_i) \right| + \left| \frac{\|\Delta K_n\| \PP_{u^*_i}^y}{d^*_i + z} \right| \right) + \sum_{i=1}^n \frac{|\PP_{u^*_i}^y|}{d^*_i + z} \|\kk\| \left(2+\sqrt{2- 2\Theta_i} \right)  \\
    % &\leq \sum_{i=1}^n|b_i|\left(\left|\sum_{k: k \neq i}\PP_{u^*_k}^y \varepsilon \right| + \left| \PP_{u^*_i}^y (1 - \Theta_i) \right| + \left| \frac{\|\Delta K_n\| \PP_{u^*_i}^y}{d^*_i + z} \right| \right) + \sum_{i=1}^n \frac{|\PP_{u^*_i}^y|}{d^*_i + z} \left( \|\kk\| \sqrt{2- 2\Theta_i} \right)  \\
    &\leq C_j\sum_{i=1}^n|\PP_{u^*_i}^y|\left( \sum_{k:k\neq i}\frac{2\|\Delta K_n\|}{|d^*_i-d^*_k|} + (1 - \Theta_i) + \frac{\|\Delta K_n\| + 2 +\sqrt{2- 2\Theta_i} }{d^*_i + z} \right),\\
\end{align*}     
where in the last inequality, we denote 
% $\sum_i^n \sum_{k\neq i} |\PP_{u^*_k}^y \varepsilon| \leq \sum_i^n n\varepsilon |\PP_{u^*_i}^y|$, and 
$C_j := \max\{|b_1|, \cdots, |b_n|, \|\kk\|\}$. Then, the claim follows naturally. 
% Since $\varepsilon$ measures the misalignment of the eigenvectors, we only care for it when $\PP_{u^*_i}^y \neq 0$, \ie it measures how much we mistake the wrong eigenvectors for the correct ones, and we assume the uniform bound of $\varepsilon$ for the simplicity of the formulation. 
\end{proof}
\begin{remark}
    Although the \cref{eq:bound} might seem complex, its structure is actually quite simple: each term is essentially the projection of the learning targets multiplied by the corresponding estimation error for that eigenvector. We will verify this empirically in \Cref{sub:exp_bound}: the magnitude and volatility of the generalization error increase from the top eigenvectors to the trailing ones. 
\end{remark}

\Cref{thm:gen} shows a given fixed $\|\Delta K_n\|$, which measures the difference between various kernel matrices formed by each sample of fixed size $n$. This difference is usually determined by the nature of the data obtained from the sampling process. Our results suggest the following: \textit{1) given the difference between empirical kernel matrices, the learning targets aligned with the top eigenvectors can generalize better; 2) alignments with eigenvectors corresponding to isolated eigenvalues can further help with generalization.} The latter observation also coincides with the empirical findings from \cite{martin2020}, which the authors noticed that well-generalized neural networks have weights that have isolated eigenvalues. 

% \imp{remark on regularization, z is the same, high z reduce the difference between kernels, the bound is augmented for large z}
\begin{remark}
    \Cref{thm:gen} assumes the same regularization between $K_n$ and $K_n^*$. We notice that by increasing the regularization parameter $z$, the differences among various kernels will diminish. Additionally, a large $z$ will cause the generalization error to dominate the approximation error. Such a large regularization will cause $K_n$ to behave like an identity matrix, which can lead to near-zero reconstruction but catastrophic generalization, as shown in \Cref{sec:rege}.
\end{remark}
% \begin{theorem}
%     Under KRR assumptions, if the sample kernel error w.r.t the optimal sample $\Delta K_n_n = K_n - K_n^*$, where $K_n^*$ is the optimal sample of size $n$. 
%     \imp{error in terms of $\Delta K_n$?}
% \end{theorem}
    
\begin{remark}
Since the top eigenvectors and eigenvalues play a crucial role in predictions, and the sample top eigenvalues and eigenvectors are known to be inconsistent, it is possible to use shrinkage functions to reduce the estimation error of the kernel matrix. The shrinkage estimation for the kernel matrix has been proposed by~\cite{lancewicki2019}.
\end{remark}

% Furthermore, we show that in the non-parametric regime, the training error has no predictive power over testing or generalization error, in the sense of the following. 
% \begin{proposition}
%     Under non-parametric KRR assumptions, suppose that the reconstruction error with regularization parameter $z$ is $\err_{t,z}$, the generalization error $\err_{p}$, for any constant $c$, there exist a $z$, we have $\err_{p} > c \err_{t, z}$, 
% \end{proposition}
% \imp{make this into a remark}

% \subsection{Bias in KRR predictions}

% Let $K$ be the kernel matrix of $n$ training samples, and $\tilde{K}$ be the kernel matrix of $n$ training samples with one additional test data point. Then, we can easily observe that $K = \tilde{K}_{-1, -1} =: M_{1}$, where we assume that the test data point is placed at the first position. The prediction of the test data point is made with the solution $A$ of the KRR using $K$ with first row vector $v_1$ of $\tilde{K}_{2:n, 1}$. 

\subsection{A simple empirical estimator}

The theoretical error bound for $\E\err_{ge}$ establishes a rigorous guarantee for the estimation error; however, it depends entirely on inaccessible optimal quantities, namely the true eigenvalues $d_i^*$, the population eigenfunctions $u_i^*$, \etc To demonstrate this bound empirically, we construct a simply empirical estimator $\widehat{\err}$ that captures the core structural factors that will influence the predictability of the learning targets.

We replace the unobserved optimal quantity $K^*$ with the empirical one $K$ constructed from the training samples. The unobservable population projection $|\mathbb{P}_{u_i^*}^y|$ is substituted with its empirical counterpart, and dropping other non-computable components, the foundational shape of the estimator becomes a weighted sum of the inverse empirical eigenvalue gaps:
$$
\widehat{\err} \propto \sum_{i=1}^n |\mathbb{P}_{u_i}^y| \sum_{k: k \neq i} \frac{1}{|d_i - d_k| + \tau}
$$
where $\tau = 1\times10^{-3}$ is a small number to prevent dividing by zero. While $\widehat{\err}$ accurately captures the trend and structural stability of the learning problem, its raw magnitude might not match actual estimation error due to dropping non-computable terms. We remedy the magnitude mismatch by a factor based on the learning target:
\[
\gamma_n = \frac{\|y\|_2}{n}
\]
Hence,  our simple empirical error estimator is the following:
\begin{equation}\label{eq:ours}
\widehat{\err} = \frac{\|y\|_2}{n} \sum_{i=1}^n |\mathbb{P}_{u_i}^y| \left( \sum_{k: k \neq i} \frac{1}{|d_i - d_k| + \tau} \right).
\end{equation}

\subsection{Relation between reconstruction and generalization}

Reconstruction error can be arbitrarily small as long as the learning target lies in the span of the kernel's eigenvectors (\Cref{prop:recon}). As we have suggested earlier, this result in non-parametric KRR is algebraic and does not  rely on data distribution assumptions. From this viewpoint, overfitting is more subtle than near-zero reconstruction (training) error, and near-zero reconstruction error also does not imply overfitting. This is consistent with the literature on benign overfitting~\cite{bartlett2020}.
% Hence, the ability for reconstruction error to predict generalization error depends on the extent to which the bias caused by regularization contains information regarding the generalization behavior of the model. 
% Our characterization of KRR shows that regularization only meaningfully affects model predictions via trailing eigenvectors, as in \cref{eq:bound}. 
% These results do not rely on data distribution assumptions. %We further show that these conclusions can also be derived using existing literature, despite some misconceptions in the literature. 
% Imposing constant regularization, \ie in the form of $z I$, cannot improve performance (besides avoiding numerical issues) for the simple matter that there is no information gain by adding the regularization. Hence, we argue that the effects of regularization and the predictive power of reconstruction error (or more traditionally, training errors) share the same roots. 

Our assertion that reconstruction or training error cannot in general predict generalization error contradicts some existing claims in the literature. In particular, \citet{jacot2020} derived a fixed relationship between training error and generalization error under restrictive assumptions of a Gaussian Process and in a large-$n$ limit. Our results in \Cref{sub:exp_bound} show empirically that generalization error can be much worse than is predicted by a fixed scaled version of training errors, \ie $\frac{\vartheta^2}{z^2} \err_{re}$, especially when targets align with the trailing eigenvectors. To predict the generalization errors, \citet{jacot2020} proposed the Kernel Alignment Risk Estimator (KARE)\footnote{
We have adjusted the original notations for consistency and brevity, \eg we use $\err_{ge}$ for $R^{\epsilon} \left( \hat{f}_{z}^{\epsilon} \right)$, and $\err_{re}$ for $\hat{R}^{\epsilon} \left( \hat{f}_{z}^{\epsilon} \right)$ in \cite{jacot2020}. We also omitted the arguments in these functions to avoid clutter. We refer the readers to the original paper for details.}: 
\begin{align}
\err_{ge} \approx \rho(z, n, y, K_n) 
&= \frac{\frac{1}{n} y^\top \left( \frac{1}{n} K_n + \lambda I_n \right)^{-2} y}{\left( \frac{1}{n} \text{Tr} \left[ \left( \frac{1}{n} K_n + \lambda I_n \right)^{-1} \right] \right)^2}   \tag{\cite[pg. 8]{jacot2020}} \\
&= \lambda^2y^\top(\frac{1}{n} K_n + \lambda I_n)^{-2}y \frac{1}{\frac{\lambda^2}{n} \text{Tr}\left[ \left( \frac{1}{n} K_n + \lambda I_n \right)^{-1}  \right] } \nonumber \\
&= \sum_{i=1}^n \left(\frac{y^\top u_i \lambda}{\frac{d_i}{n}+\lambda}\right)^2 \frac{1}{\frac{\lambda^2}{n} \text{Tr}\left[ \left( \frac{1}{n} K_n + \lambda I_n \right)^{-1}  \right] } \nonumber \\
% &={\color{red}\text{TO-DO: FIX BELOW (lambda to z)}} \\
&= \err_{re} \frac{1}{\frac{\lambda^2}{n} \text{Tr}\left[ \left( \frac{1}{n} K_n + \lambda I_n \right)^{-1}  \right] } \tag{\Cref{prop:mse}} \\
&=: \err_{re} \cdot \phi(\lambda, X) \nonumber
\end{align}
where $K$ is the Kernel/Gram matrix, $\lambda$ is the ridge regularizer used by \citet{jacot2020} (related to ours by $z = n\lambda$), $\phi$ is a function of $z$ and $X$, $y$ is the learning targets. Since $\err_{re}$ depends on the kernel $K$, and in turn, depends on input data $X$, $\E\err_{ge} \neq \E \err_{re} \cdot \E\phi(\lambda, X)$. This means expected risk cannot be written as a scaled version of expected training error; in fact, the gap between the true error and the error predicted by this scaling law is precisely the covariance between $\err_{re}$ and $\phi(\lambda, X)$. 
% In appendix \ref{app:jacot}, we provide further analysis of the results of \citet{jacot2020}.
Moreover, Theorem 6 and Theorem 7 of \citet{jacot2020} suggest that
\begin{align*}
    |\E (\err_{ge}) - \frac{\vartheta^2}{z^2} \E(\err_{re})| &\leq |\E(\err_{ge}) - \tilde{R}^{\epsilon}| + \left| \frac{\vartheta^2}{z^2}\E(\err_{re}) - \frac{\vartheta^2}{z^2} \frac{z^2}{\vartheta^2}\tilde{R}^{\epsilon}\right| \\
    &\leq \left(\frac{z^2C+\vartheta^2}{z^2n} + P_2\left(\frac{\tr(T_K)}{zn^{1/2}} \right) + \frac{\vartheta^2}{z^2}P_3\left(\frac{\tr(T_K)}{zn}\right)\right)\tilde{R^\epsilon}
\end{align*}
where $\vartheta$ is the signal capturing threshold, $P_2, P_3$ are polynomials with nonnegative coefficients, such that $P_2(0) = P_3(0) = 0$, $T_K$ is the covariance operator, $\tilde{R}^\epsilon$ is also an approximator of expected risk (generalization error) $\E(\err_{ge})$. For finite $n$, the bound on the right may be non-negligible: there is no guarantee that that the polynomials $P_2, P_3$ are close to zero. This coincides with our analysis in \Cref{prop:recon}, in contrast to the original claim from \citet{jacot2020}, that the training errors will deviate from generalization errors when generalization errors are far from zero. In other words, training errors alone cannot predict generalization errors except in regimes where generalization error is extremely low. Our experiments in \Cref{sub:exp_bound} also validate this observation.

\section{Related Work}

\citet{cristianini2001} proposed the idea of kernel alignments, suggesting that when the kernel aligns with the labels, the model will have better generalization performance. They provide a measure of alignment based on the kernel and the learning target directly without providing eigenvector-based analysis on generalization behavior.
\citet{jacot2020} shows that kernel alignments provide critical insights into the expected risk of the model. In particular, they proposed the concept of a signal capture threshold and demonstrated the relationship between training error and expected risk in closed form. Furthermore, while \cite{jacot2020} did not make assumptions about the data, they made an implicit Gaussian assumption for the sampling operator. 
\citet{simon2021} also proposes a similar eigenlearning framework and introduces the concept of a conservative law of learning, suggesting that the learnability of the model is equal to the number of observations. 
\citet{canatar2020} utilizes replica methods to study the spectral bias of kernel models and their application in neural networks. \citet{canatar2023} extends the spectral theory on kernel models to study the connection between the encoding models of neural responses and the geometric properties of deep neural network (DNN) representations and neural activities. 
Also utilize replica methods, \citet{bordelon2020} derives the generalization bounds for kernel regression as a function of the number of training samples. \citet{barzilai2023} points out that much of the existing literature uses impractical assumptions that do not hold in practice. However, the bound provided by \citet{barzilai2023} rely on choosing an arbitrary parameter $k$, and many of the results are only bounded by an order of magnitude. 

The literature also examines the generalization of neural networks through kernel alignments.
\citet{canatar2020} analyzes the spectral bias of neural networks and demonstrates that the neural network training process can be decomposed using the eigenfunctions of the neural tangent kernel. \cite{canatar2020, simon2021, bordelon2020} also connects the kernel alignments to the neural network through the wide neural network limits and the neural tangent kernel. \citet{wei2022} shows that the classical generalized cross-validation estimator can be extended to the overparameterized regime, such as in neural networks, through the framework of random matrix theory. 

Recent studies, \eg \cite{hastie2022, liang2018}, show that high-dimensional ridgeless regression models can interpolate while still exhibiting good generalization. \citet{hastie2022} also shows that the optimal ridge/regularization also vanishes in this regime, which is similar to our observations. 

Our analysis focuses on the relationship between generalization and the estimation of eigenvalues and eigenvectors. Many efforts have been put into studying the eigenvalue and eigenvectors of the covariance matrix, \eg \cite{loukas2017, silverstein1990, mestre2008aa}, \textit{etc.}

\section{Experiments}

In this section, we demonstrate our main findings through empirical experiments. 
\subsection{Eigen Alignments}\label{sec:exp1}
\subsubsection{Simulated Data}
\begin{figure}[ht]
    \centering
    \includegraphics[width=0.75\linewidth]{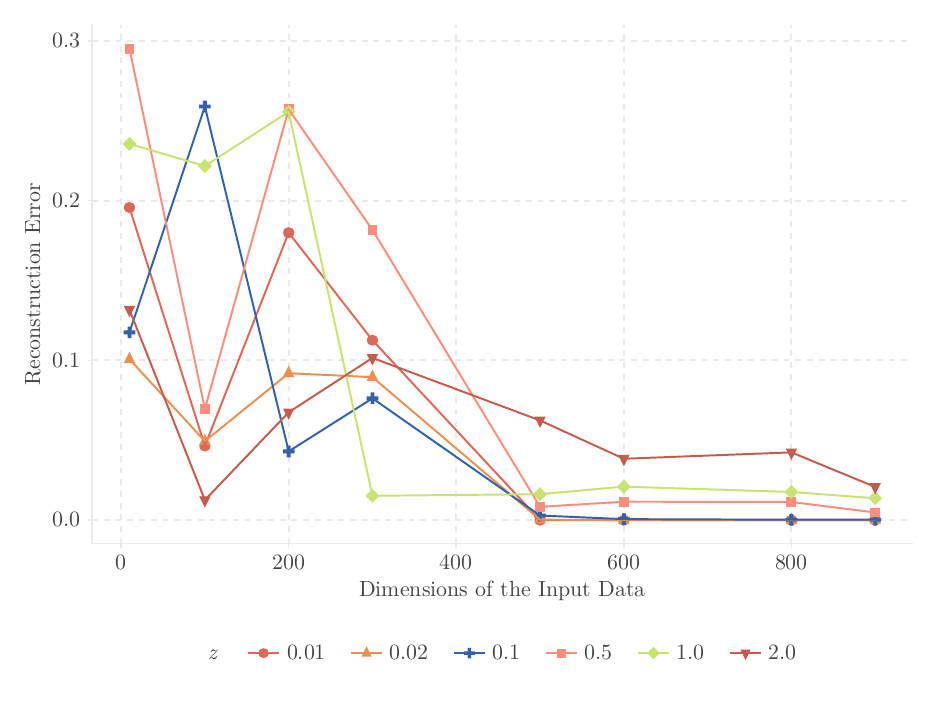}
    \caption{Random data reconstruction error. We observe the following: 1) high reconstruction errors only appear when the dimension is not high enough (\ie less than 500) for small regularizations; 2) when dimension $n = 600$, only the three highest regularization cases have reconstruction error significantly larger than zero. }
    \label{fig:dim_z}
\end{figure}
% \begin{figure}[ht]
%     \centering
%     \includegraphics[width=0.8\linewidth]{plots/align_top_trailing.pdf}
%     \caption{Training and testing error comparison for different learning targets. The number on the right margin indicates the regularization parameter. We notice that for both training and testing, when the learning targets are aligned with top eigenvectors, the model obtains small error on both training and testing, implying that the generalization is more relevant to eigenalignment than training or reconstruction error. Additionally, we notice that regularization parameter is more relevant to trailing aligned targets; For top-aligned targets, when larger regularization ($z \geq 0.5$), can still slightly reduce test error, but at the highest dimension, the larger regularizations make test errors larger.}
%     \label{fig:align_top_trailing}
% \end{figure}

% We empirically demonstrate the eigen alignment phenomenon: utilizing the synthesized targets, we show that test error increases when the learning targets are aligned with trailing eigenvectors, despite the nearly perfect reconstruction error. 
We first demonstrate that reconstruction errors in high dimensions can be arbitrarily small, as suggested in \Cref{prop:recon}, with any bias attributable to the regularization parameter $z$. To this end, we randomly sample input data $X$ with a dimension of $p\times 1000$. We fix the number of observations $n = 1000$, and vary the dimension of the input data $X$; specifically, $p = 10, 100, 200, 300, 500, 600, 800, 900$. For the learning objectives $y$, we utilize the normalized weighted sum, with random weights, of $1st, 2nd, 100th, 400th, 500th$ eigenvectors of the linear kernel matrix $X^\top x$. We calculated the sum of the squared reconstruction error 
$$
n\err_{re} = \sum_i (\hat{y}_i - y_j)^2.
$$
as our measurement.
We plot the reconstruction errors in \Cref{fig:dim_z}. It shows that as the dimension increases, we can achieve a near-zero reconstruction error, influenced by the bias introduced by the regularization parameter $z$. Notice the sharp decline that occurs after the dimension of the input data $p \geq 500$. In the case of our choice $y$, since it relies on the trailing eigenvectors, it is unlikely to generalize well. However, a zero reconstruction error does not necessarily indicate poor generalization or overfitting, as we demonstrate in the following sections.

% Next, we demonstrate that, provided that the data align with the top eigenvectors, good generalization can still be achieved even with near-zero reconstruction errors. We continue to utilize randomly generated data as previously. We keep the number of observations at $1000$ in different runs and ensure that the data remain the same across these runs as well. For training, we first take $700$ observations, reserving the rest for testing purposes. The results are plotted in \Cref{fig:align_top_trailing}. We clearly show that a small training error, even close to zero, does not indicate overfitting. Furthermore, when learning targets are aligned with the top eigenvectors, the model obtains smaller errors in both training and testing. Additionally, we notice that for the learning targets that are aligned with top eigenvectors, the regularization exhibits little effect. In particular, \textit{we observe that, in high dimensions, regularization can adversely affect generalization performance: } for top-aligned targets, when $n = 600$, larger regularization ($z \geq 0.5$) may still slightly reduce test error. However, in the highest dimensions, increased regularization results in larger test errors.
\begin{figure}[!ht]
    \centering
    \includegraphics[width=0.75\linewidth]{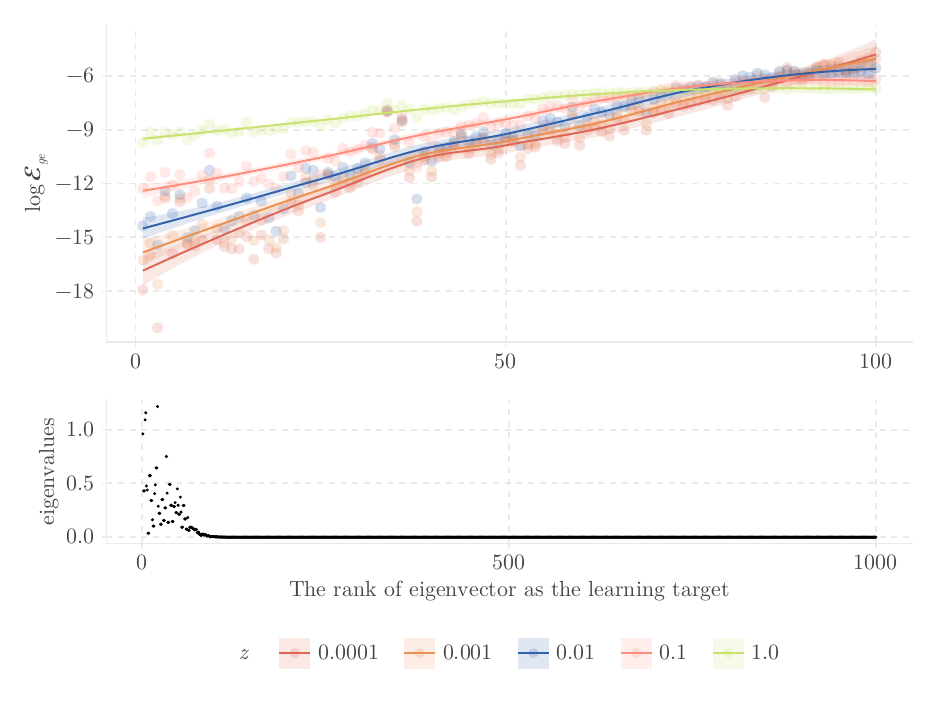}
    \caption{Learning targets of the task affect the generalization performance. Learning targets aligned with top eigenvectors (lower rank) tend to have lower generalization errors. Regularization tends to have a negative impact on generalization error for most learning targets, except those aligned with trailing ones (high-rank eigenvector). Additionally, eigen-gaps are largest }
    \label{fig:reg}
\end{figure}

Next, we demonstrate the effects of regularization in non-parametric KRR; we randomly sample 1000 data points from a uniform distribution from $(-3, 3)$, \ie each data point is scalar in $(-3, 3)$. Then, we use an RBF kernel to construct the kernel matrix and calculate the eigenvalues and eigenvectors. For training purposes, we randomly sampled 100 data points from the 1000 available data points. We use each eigenvector we calculated (also sampled to 100 points) as the learning targets for training.  
To see how both learning targets and regularization affect generalization, we used trained models to predict the entire data set and calculate the generalization error. \Cref{fig:reg} shows the results of this experiment. To better present the difference, we plot the generalization error on a logarithmic scale. It is clear from \Cref{fig:reg}, that the generalization error increases as the rank increases, \ie that the learning targets align with the trailing eigenvectors. Additionally, we notice that for most of the learning targets, the larger the regularization, the larger the generalization error. This trend only reverses towards the trailing eigenvectors. At the bottom of \Cref{fig:reg}, we also show that large eigen gaps also correspond to the smaller generalization errors.

Our experiments suggest that regularization is more effective when learning targets are aligned with trailing eigenvectors. This coincides with our argument in the previous section: if the regularization is effective, the model is already in the unstable estimation regime. This is particularly true in higher dimensions, as illustrated in the following section, specifically, in \Cref{fig:multi}. Note, we only consider the magnitude, not the variance of the error, where the regularization might be beneficial; however, the detailed analysis on variance is left to future work.  

\subsubsection{Real Datasets}\label{sub:overfit}

Our experiments in this section utilize more realistic data to demonstrate our results, which are also applicable to real-world scenarios. 
For this experiment, we sampled 1000 images from the MNIST dataset, utilizing only those labeled 4 and 9. Despite the perceived simplicity, these data are actually of high dimensionality. We constructed a learning target based on the eigenvectors of the kernel matrix using a squared inner product kernel $\kfn(x_i, x_j) = (x_i^\top x_j)^2$. For comparison, we construct two sets of learning targets: one aligned with the top eigenvectors, $y_1 = u_3 - u_5$, where $u_1$ and $u_3$ are $3rd$ and $5th$ eigenvectors, respectively; and one aligned with the trailing eigenvectors, $y_2 = u_{100} + u_{101}$. The results are shown in \Cref{fig:eigentarget}.

\begin{figure}[ht]
    \centering
    \includegraphics[width=0.95\linewidth]{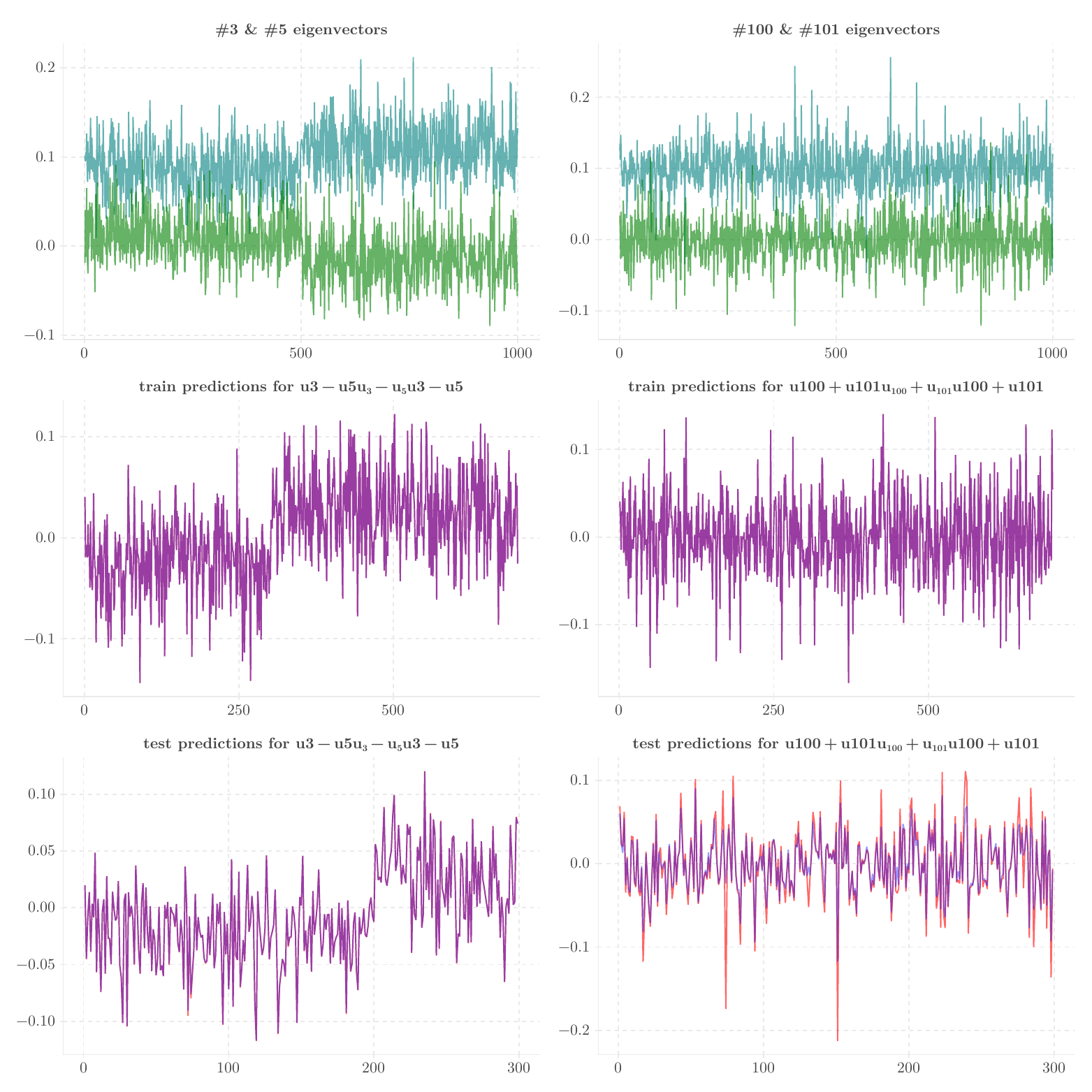}
    \caption{With MNIST images data (labeled 4 and 9), we construct targets with top and trailing eigenvectors. On the left, we see both training and testing have near-perfect performance (the blue prediction line overlaps with the red ground truth). On the right, we notice that although the training error is near zero, \ie near-perfect overlapping of red and blue lines, the testing error is significant.}
    \label{fig:eigentarget}
\end{figure}

On the left of \Cref{fig:eigentarget}, we observe that both training and testing error are quite low; this is shown by noticing the ground truth $y_1$ and predictions $\hat{y}_1$ are almost perfectly overlapped.  In the right column of \Cref{fig:eigentarget}, we observe that despite the almost perfect overlap of the ground truth and the prediction for training, the performance on the tests is quite poor. These results show that alignment of the top eigenvectors is the key to good generalization, and more importantly, in a nonparametric regime, the training error can be quite uninformative regarding the test performance.

% \begin{figure}
%     \centering
%     \includegraphics[width=0.75\linewidth]{plots/ours_vs_jacot_v2_trace.pdf}
%     \caption{True testing error compared with existing metrics from \cite{jacot2020}. Notice that since the plot is in log scale, the estimated error $\varrho$ from \cite{jacot2020}, and the scaled reconstruction error $\frac{\vartheta^2}{z^2}\err_{re}$ are diverging when the learning targets are aligned with trailing eigenvectors. Our proposed simple estimator captures this eigen-alignment behavior most closely. 
%     }
%     \label{fig:jacot2}
% \end{figure}

To further evaluate this phenomenon, we compare the training errors and testing errors with an additional three datasets: CIFAR10, FashionMNIST, and SVHN2. To ease the computational burden, we subset data by selecting 2 of the available labels, and we also limit the size of training samples to 1000 for each dataset. We evenly chose 10 eigenvectors and multiplied them by 10 (to increase the magnitude) as the learning target for each run and computed the training and testing errors. The results are shown in \Cref{fig:multi}.  
\begin{figure}[ht]
    \centering
    \includegraphics[width=0.85\linewidth]{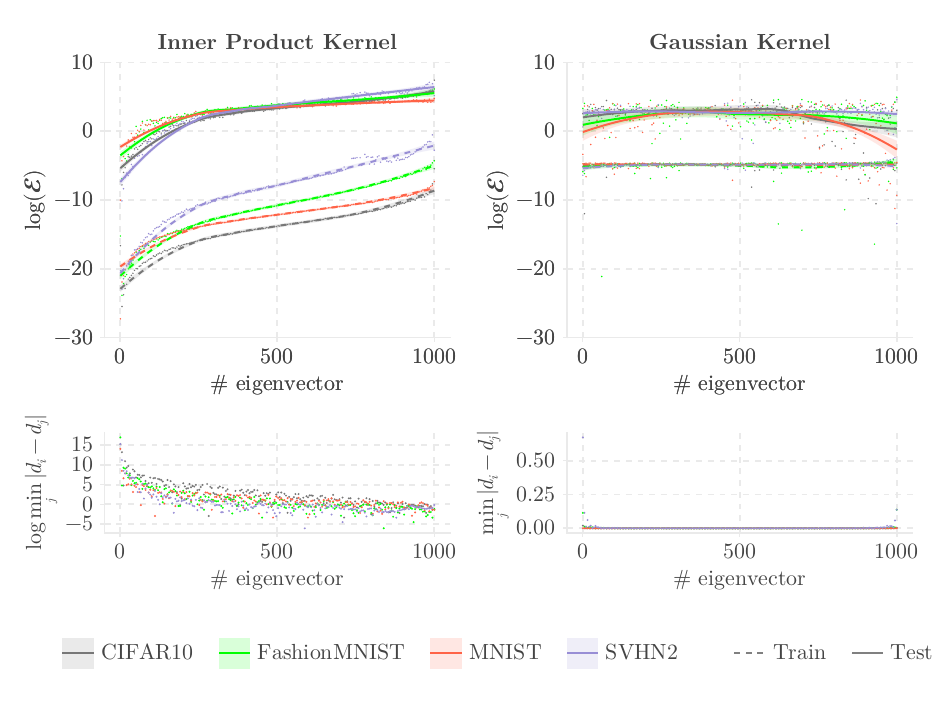}
    \caption{Training error and testing error comparison for different kernel types and eigenvectors as the learning targets. The Gaussian kernel showing a wider spread implies greater variance due to almost uniform eigenvalue distributions. Both kernels exhibit a correlation between the generalization error and the eigengaps. We used a fixed regularizer of $z=0.2$ in all the experiments.}
    \label{fig:multi}
\end{figure}
In addition to the inner product kernel, we also utilized a Gaussian kernel to compare empirical behaviors of KRR. Specifically, we observe that the eigenvalue distribution for a Gaussian kernel is almost uniform, indicated by the small eigengaps $\min_j |d_i - d_j|$ for all eigenvalues. This is expected since the $L_2$-norm-based kernel has an asymptotic rank of one when excluding the diagonal. As our analysis shows, the small eigengaps cause significant prediction variance. Since the learning targets are small in norm, the model can accidentally achieve a small error in the testing set by simply making small predictions. Meanwhile, the inner product kernel shows much more stable prediction; additionally, as the eigengaps decrease, the generalization errors increase. 

In both cases, the training errors are relatively small across learning targets. Since the trailing eigenvectors are smaller, the bias caused by the regularizations is also more significant in the trailing eigenvectors.  We see that despite the training error gradually increasing as the learning targets move to the tail of the eigenspectrum, they remain close to zero. As analyzed in \Cref{prop:recon}, the training error in non-parametric KRR can be arbitrarily small, up to regularization effects (mainly the relative magnitude of eigenvalues compared to the regularizer) and numerical error. The testing errors, however, increase significantly for the trailing eigenvectors. 

\begin{remark}
We notice that in both \Cref{fig:reg} and left-hand sides of \Cref{fig:multi} the generalization error exhibit similar trends, despite the fact that the model in \Cref{fig:reg} also uses a Gaussian kernel. However, since it is in low dimension (dim = 1 in \Cref{fig:reg}), the Gaussian kernel matrix has a more non-trivial eigenvalue distribution: after subtracting the diagonal, the kernel matrix in \Cref{fig:reg} actually has a higher rank than the Gaussian kernel matrices in \Cref{fig:multi}. Interestingly, the eigenvalue distributions of a Gaussian kernel for low-rank data actually resemble the eigenvalue distributions of inner product kernels, which have a quick decay in the magnitude of eigenvalues. This is aligned well with our theory, which identifies large eigengaps as a factor that can lead to better generalization. %Additionally, this shows that although our theory's primary focus is on high-dimensional data and inner product kernels, it also applies well to low-dimensional data with a Gaussian kernel where it still works. 

Finally, we also observe the similar ``flipping" effects on both \Cref{fig:reg} and the left-hand sides of \Cref{fig:multi}: the regularization that causes the most bias in the top eigenvector regime, leading to the most improvement in generalization in the trailing eigenvector regime. This further demonstrates the limited effects of regularization in non-parametric KRR, such that regularization is most likely to be beneficial only when the learning targets are aligned with trailing eigenvectors. 
\end{remark}

\subsubsection{Empirical bound}\label{sub:exp_bound}

To further validate our generalization bound, we empirically demonstrate that our simple empirical estimator, as shown in \cref{eq:ours}, is capable of capturing generalization behavior of KRR. 
\begin{figure}[ht!]
    \centering
    \includegraphics[width=0.99\linewidth]{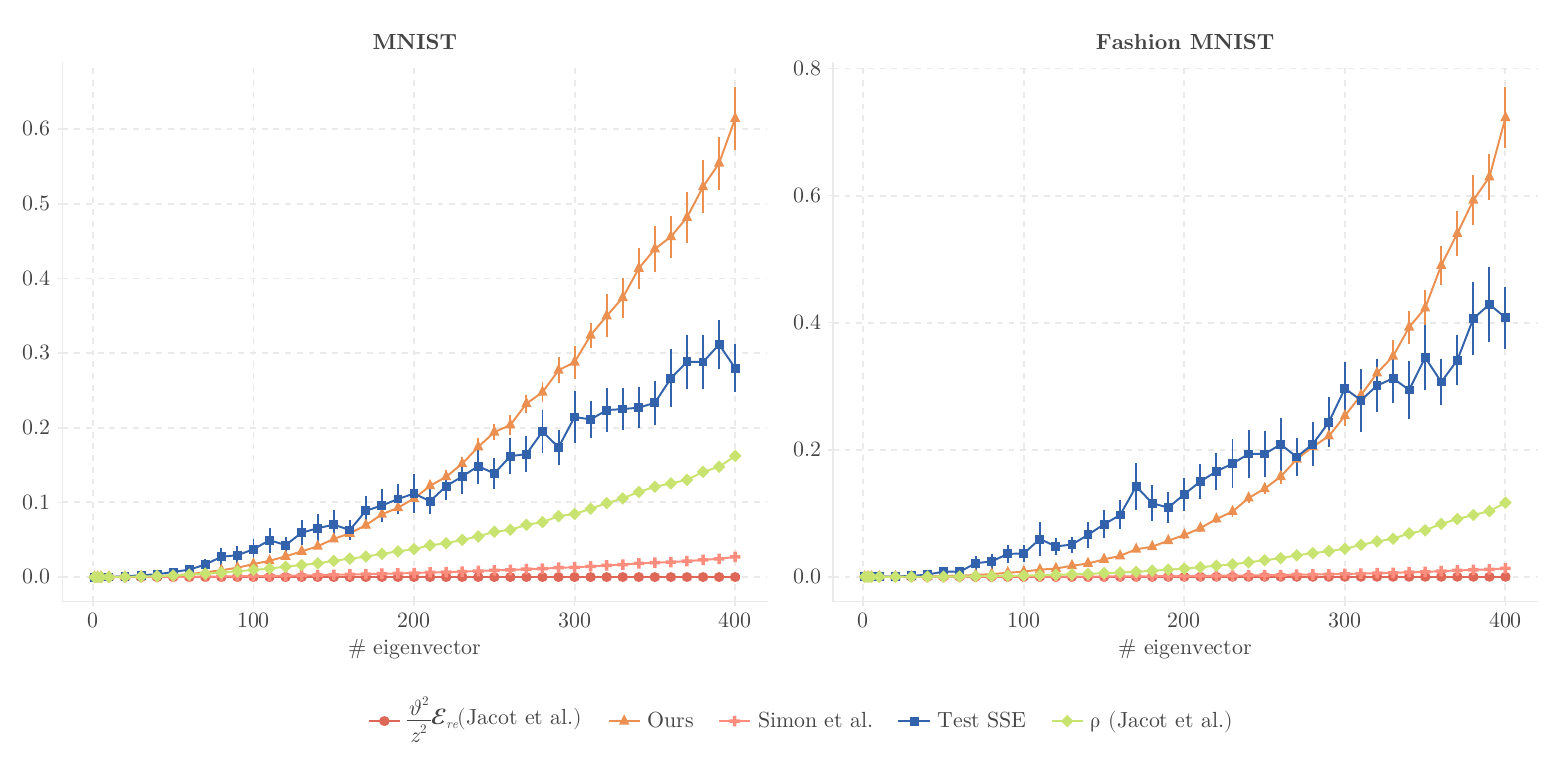}
    \caption{True testing error compared with existing metrics from \cite{jacot2020, simon2021}. Notice that the estimated error $\rho$ from \cite{jacot2020}, and the scaled reconstruction error $\frac{\vartheta^2}{z^2}\err_{re}$ are diverging when the learning targets are aligned with trailing eigenvectors. Our proposed simple estimator captures this eigen-alignment behavior most closely, especially in the top eigen-directions. Existing literature significantly underestimates prediction errors in trailing eigen-directions. Eigen learning targets were computed using a subset of 1,000 images, and performance metrics were estimated by resampling the training set 100 times. The plot reports the mean, with error bars indicating $\pm 1$ standard deviation across random data splits.
    }
    \label{fig:jacot}
\end{figure}
In the experiments, we compare our estimator with the Kernel Alignment Risk Estimator $\rho$ from \citet{jacot2020}, and the test MSE estimator from \citet{simon2021} in terms of predicting the sum-of-squared errors (SSE) for testing data (MNIST dataset and FashionMNIST dataset, same setup as \Cref{fig:multi}), as these are most comparable to our work. We adjust all the formulas accordingly to calculate the sum of squared errors of the testing data points. The results are shown in \Cref{fig:jacot}. Additionally, we also plot $\frac{\vartheta^2}{z^2}\err_{re}$ (where $
\vartheta$ is the signal capturing threshold proposed by \cite{jacot2020}), the fixed relation between the reconstruction error (training error) and the test error. Unsurprisingly, as we have suggested, the training error has little predictive power over testing error. We observe that \citet{simon2021} fails to track the risk in trailing eigendirections at all. \citet{jacot2020} also fails to track the testing errors for the FashionMNIST dataset. Our simple estimator is able to track the magnitude of the Test SSE most closely for the leading eigenvectors (roughly up to index 250 for MNIST and index 300 for Fashion MNIST).

\begin{remark}
It is worth emphasizing that our goal is not to propose an accurate estimator, since we prioritize brevity over tightness in the analysis. Nonetheless, we find that the simple estimator above can still track the test error with reasonable accuracy. Therefore, our results show a significant limitation in current literature regarding generalization errors in trailing eigen-directions. 
Note that the estimator begins to overestimate the errors in the trailing eigendirections; this is expected because eigengap-based bounds are known to be loose for the corresponding trailing eigenvectors. Our theoretical results remain valid, and sharper eigenvector perturbation bounds could be used to construct an improved estimator, which we defer to future work.
\end{remark}

We also show in \Cref{fig:100err} that our analysis indicates that when learning targets are aligned with the trailing eigenvector, not only does the generalization error increase, but the volatility also increases. In the experiment, we randomly sampled 700 images from a total of 100 to form our training set and utilized the remaining images as the testing set. We repeat each learning target 100 times to illustrate its volatility. 

\begin{figure}[ht]
    \centering
    \includegraphics[width=0.75\linewidth]{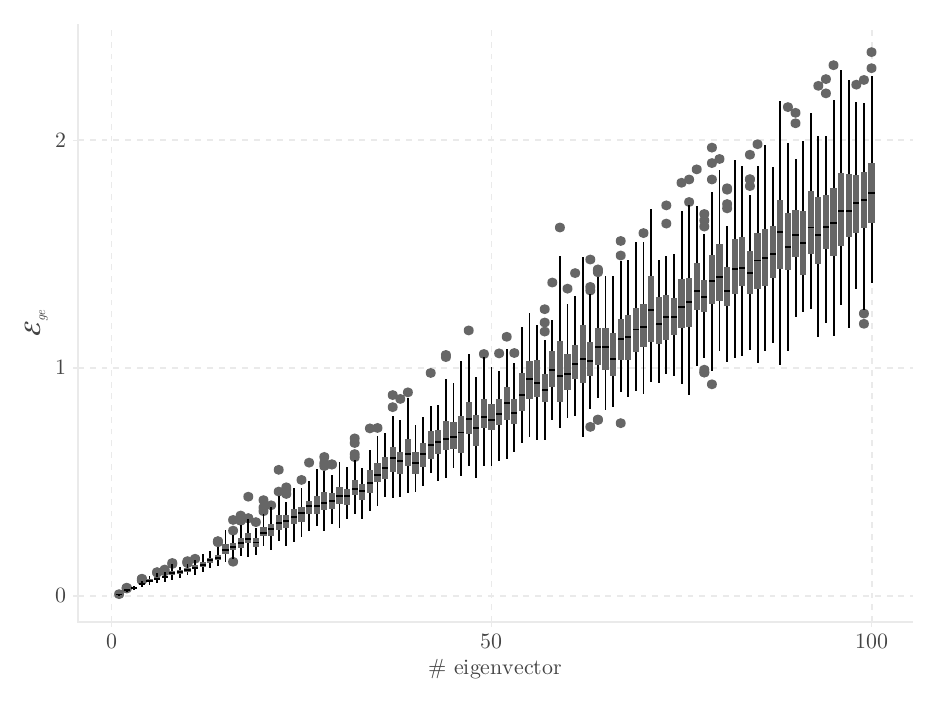}
    \caption{Generalization error when treating each eigenvector as an individual synthetic learning target.}
    \label{fig:100err}
\end{figure}

\subsection{Reconstruction error and generalization}\label{sec:rege}

Next, we will demonstrate that, according to \Cref{cor:triv}, the reconstruction error has limited predictive power regarding the generalization error. We sampled 1000 images labeled ``4" or ``9" from the MNIST dataset for our training data. By choosing the following four kernels, we demonstrate the following: 1) an improperly scaled kernel can have a small reconstruction error due to the magnified eigenvalues; 2) A near-zero training error does not indicate overfitting, as a squared kernel can exhibit both a small reconstruction error and a generalization error; 3) trivial kernels can exhibit near-zero training errors while demonstrating trivial generalization. Without loss of generality, all the kernels used are inner-product kernels. The first three kernel functions are as follows: 
\begin{align}
    f_1(t) &= t, \;\; f_2(t) = \frac{t}{\sqrt{784}}, \;\; f_3(t) =  \frac{t^2}{\sqrt{784}}.
\end{align}
Then, the $(j,k)$th element of the kernel matrix constructed with $f_i$ is $[K_i]_{j,k}:= f_i(x_j^\top x_k)$. The trivial kernel\footnote{Although this is not a valid kernel in RKHS, but in finite dimension, it is a valid way to construct kernel matrices.} is defined as $[K_4]_{j,k} = 500 \delta_{jk}$, where $\delta$ is the Dirac delta. We use $1/n \sum_{i=1}^n(y_i-\hat{y}_i)^2$ to calculate both the reconstruction error and the generalization error. All kernels were trained and tested using the same data set. 

\begin{table}[!ht]
\setlength\tabcolsep{0pt}
\centering
% \begin{threeparttable}
\caption{Reconstruction error and generalization error comparison for different kernels}\label{tbl:rege}
\begin{tabular}{@{\extracolsep{2ex}}*{5}{rcccc}}
\toprule
\textbf{} & \multicolumn{4}{c}{\textbf{Kernels}} \\
\cmidrule{2-5}
\textbf{} & $K_1$ & $K_2$ & $K_3$ & $K_4$ \\
\midrule
\textbf{Reconstruction error} & 0.08 & 0.106 & 4.6e-6 & 4.0e-8 \\
\textbf{Generalization error} & 0.382 & 0.279 & 0.193 & 1.99 \\
\bottomrule
\end{tabular}
% \end{threeparttable}
\end{table}

From \Cref{tbl:rege}, we observe the following: 1) the improperly scaled kernel $K_1$ exhibits a smaller reconstruction error in comparison to the properly scaled counterpart $K_2$, due to larger eigenvalues; however, it generalizes worse; 2) $K_3$ exhibits both a smaller reconstruction error and a generalization error, suggesting that a near-zero reconstruction error does not indicate overfitting or poor performance; 3) $K_4$ also exhibits a near-zero reconstruction error due to high eigenvalues; however, it demonstrates poor generalization. This indicates that reconstruction is a trivial task with minimal predictive power in generalization.

\section{Discussion: Limitations and Comparison}

% \textcolor{red}{TODO: no noise! not even homoskedastic! Main analysis is on K and $\Delta K_n$, which depends only on $x$ (and variability in sampling process); noise would affect targets $y$, which affects only one term in the bound in eq (14). Future work?}

% \textcolor{red}{TODO: see typst stuff on data imputation and missing values.}

One limitation is that we assume that observed labels $y$ are identical to the ground truth; that is, we assume no noise in the targets. Based on our analysis, one may be able to bound the effects of (either homoskedastic or heteroskedastic) noise by considering the projection of the noise onto (mis)estimated eigenvectors. A second limitation is that our analysis studies the problem of learnability in terms of observed information, though in real-world settings data may contain missing values. It is possible that the standard practice of data imputation (using a kernel) could be incorporated into our analysis, \eg as additional structured noise in the perturbed $\Delta K_n$. A third limitation is that our generalization bound in \cref{eq:errbound} was derived from a sequence of inequalities, which may not be tight in practice; in this sense, our results are an initial exploration of how perturbation analysis can be applied to KRR, and we expect that further developments in perturbation analysis will be able to sharpen these bounds. All of these limitations point to interesting directions for future work.

Our work differs from several notable existing studies, such as \cite{jacot2020, canatar2020, bordelon2020, simon2021}, in terms of the assumptions and perspectives that underpin our analysis. We primarily focus on the effects of finite sample sizes on the generalization error, while the existing literature addresses asymptotic behaviors in infinite dimensions, typically through Gaussian process-related tools. We aim to provide a fresh approach to kernel alignment analysis with finite matrix perturbation theory. As we have stated earlier, we prefer the brevity of the results to illustrate the qualitative insights our analysis can provide, where the tightness of our bounds is likely to be further refined. However, we will defer the sharpening of the bound, along with further investigation into more structured kernel matrices, to future work. 

Existing work on the analysis of generalization, \eg \cite{smale2005, simon2021, jacot2020}, tends to treat the function as an infinite-dimensionall vector, with the training data viewed as random samples drawn from a Gaussian process. However, this assumption overlooks the complexity and richness of the information embedded within the training data and the learning targets. 
%This omission contributes to the misconception regarding the relationship between reconstruction error and generalization error. Additionally, by considering infinite dimensions, the misaligned issue of predictions in non-parametric KRR vanishes, as suggested by our analysis. This is where our work contributes to bridging the existing gap. 
We found that generalization fundamentally relies on the consistent estimation of ``useful'' information, \ie eigenvalues and eigenvectors that are aligned to the targets, particularly when dealing with finite sample sizes.
In contrast, in \cite{jacot2020}, the generalization error and the training error are correlated by a factor $\vartheta(\lambda^2)/\lambda$, which is independent of the learning targets. 
We have demonstrated, through both our analysis and numerical experiments, that this is not the case; we utilize synthesized learning targets to demonstrate that the ratio of generalization error to training error depends strongly on the learning targets. \citet{simon2021} also provided a similar ratio called ``overfitting coefficient'', which suffers from a similar issue. %The primary reason for the discrepancy in these results, in our view, is that they blurred the distinction between learning targets and the eigenmodes to be learned due to their Gaussian process setup. While these methods can still provide insights into generalization, they will fail in various cases, \eg the kernel is ill-formed, as we show in \Cref{cor:triv} and in the experiments conducted in \Cref{sec:rege}. the reconstruction error can be arbitrarily small and is unlikely to serve as a reliable indicator of generalization error, rendering the correlation practically misleading. 

% Our approach also provides an upper bound analysis for worst case scenarios. Compared to the expectation approach taken by existing literature, \eg \cite{simard2000,jacot2020}, our analysis reveals a potential risk that was obscured by the expectation approach, namely, volatility. Despite in theory, given enough samples, we can learn the trailing eigenvectors; in the finite sample case, we cannot ensure good performance. This risk is revealed through the utilization of tools from matrix perturbation analysis. 
% By assuming infinite dimensions for the data, \eg \cite{bordelon2020, canatar2020}, it essentially eliminates the challenges associated with finite samples, particularly concerning the estimation of trailing eigenvectors. These analyses primarily address the inherent randomness from noise, rather than the randomness caused by imperfect samples. Hence, it is reasonable to expect that regularization will be effective in this regime, where the perturbation effects of finite sampling diminish. 

Our results also provide insights into the limitations of regularization for Kernel Ridge Regression (KRR) from the perspective of a finite kernel matrix. From our previous analysis and \Cref{rm:reg}, we show that $d_i + z \ll \PP_{u_i}^yd_i$ is necessary for the eigenvector to be regarded as relevant. Additionally, the eigenvalues of the top eigenvectors typically increase as the size of the kernel matrix grows, \ie with an increase in the number of observations. This implies that a higher regularization is more permissible. The small differences induced by regularization tend to increase as the learning targets align with the trailing eigenvectors. These observations are aligned with \citet{bordelon2020}, who suggest that explicit regularization might hinder performance when regularization is large relative to the sample size. 

Finally, like \citet{simon2021}, our work differs from that of \citet{canatar2020} and \citet{bordelon2020} in the sense that we derived an explicit expression for the generalization error, \ie without relying on implicit expressions. In our case, this comes from using a matrix perturbation perspective rather than random matrix theory as used by \citet{jacot2020}, allowing us to make statements about the finite sample regime. As we have shown in the \Cref{sub:exp_bound}, our results reveal that existing literature significantly underestimates the generalization errors in trailing eigen-directions. 
% Additionally, similar to \citet{simon2021}, our results offer greater interpretability than those of \cite{canatar2020, bordelon2020}. In contrast to \citet{simon2021}, we do not rely on any newly introduced concepts; instead, we derive the generalization bound directly from a matrix perturbation perspective. 

% Additionally, by assuming the ground truth data as Gaussian process, authors also implicitly assume the ordering relationship between data observations, the level of similarities of the data observations. 

\section{Conclusion}

Our study departed from the conventional functional space view, focusing instead on how the specific choice of learning targets significantly impacts a model's generalization behavior, particularly in the context of Kernel Ridge Regression (KRR). This approach allows us to meticulously dissect the intricate relationships among reconstruction error, overfitting, and robust generalization. Our findings show that generalization depends on the model's ability to extract information from data consistently, and different learning targets can lead to significant differences in generalization performance. In contrast to previous work, we observed that the reconstruction error often has minimal predictive power for generalization, as a near-zero training loss is nearly universally achievable. We demonstrate this using synthesized learning targets: as we change the learning targets, the generalization error changes drastically, while the reconstruction error remains relatively constant. This is because the reconstruction error is independent of the stability of the information. In Kernel Ridge Regression (KRR), this phenomenon arises from the perfect alignment of the eigenvectors of two kernel matrices in the KRR prediction formula $\hat{y} = K^\prime K A$. 

In passing, we also notice that the regularization parameter frequently has a very limited positive impact in non-parametric KRR, as we show in the analysis and experiments. Regularization is more likely to have a positive impact when the learning targets align with trailing eigenvectors, but aligning with trailing eigenvectors is already a sign of poor generalization. 
% ; its primary function, we argue, is to mitigate numerical instability rather than to reduce the hypothesis space. When regularization does help with generalization, it is usually to salvage a disaster rather than to boost a success. \ie when regularization influences generalization, it is specifically due to its alignment with trailing eigenvectors, which occurs when the corresponding eigenvalue is comparable in magnitude to the regularization parameter.
We proposed bounding the generalization error through the lens of eigenvalue and eigenvector estimation. We show that good generalization requires either increasing alignment of eigenvectors or an increase in the eigenvalue; both of these can be achieved if the learning targets are aligned with the top eigenvectors. Based on our generalization error bound, we are able to derive a simple empirical risk estimator, which we have shown to track generalization errors better than existing literature in trailing eigen-directions. 

% With limited assumptions regarding kernel matrices, our generalization risk bound is intended more for qualitative insights than for practical application; thus, we reserve further refinements of the risk bound for future work. 

% Finally, our empirical investigations into the relationship between eigenalignment and the double descent phenomenon revealed that this intriguing behavior can manifest when target variables exhibit a correlation with trailing eigenvectors. Collectively, these insights underscore the critical and often understated importance of eigenalignments, offering a more comprehensive and nuanced understanding of generalization in KRR.
\bibliographystyle{plainnat}
\bibliography{paperpile}% common bib file

\end{document}